\documentclass[twoside,11pt]{article}

\usepackage{jmlr2e}

\usepackage{amsmath}
\usepackage{url}
\usepackage{psfrag,epsf, amsmath}
\usepackage{enumerate}
\usepackage{booktabs}
\usepackage{tikz}
\usepackage{listings}
\usepackage{enumitem}
\RequirePackage{pdf14}
\usepackage{ulem}
\usepackage{caption}
\usepackage{subcaption}

\usepackage{mathtools}

\DeclarePairedDelimiter\floor{\lfloor}{\rfloor}
\usepackage[T1]{fontenc}

\usepackage{epstopdf,gensymb}
\usepackage{graphics}
\usepackage{color}
\newcommand{\erdos}{Erd\H{o}s-R\'enyi }
\usepackage{amsfonts, animate}
\usepackage{amssymb}
\usepackage[margin=1in]{geometry}
\usepackage{setspace}
\usepackage{indentfirst}
\usepackage{float}
\usepackage{amsfonts}
\usepackage{bm,pbox}
\usepackage{array}

\usepackage[utf8]{inputenc}
\usepackage[english]{babel}

\usepackage{multirow}
\usepackage{color}
\usepackage[margin=1in]{geometry}
\usepackage{framed}
\usepackage{amsfonts,amsbsy,latexsym,dsfont}
\usepackage{enumerate}
\usepackage{colortbl}

\ShortHeadings{Principal Component Analysis for Networks}{Wilson and Lee}
\firstpageno{1}

\begin{document}

\title{Interpretable Network Representation Learning \\ with Principal Component Analysis}

\author{\name James D. Wilson \email wilsonj41@upmc.edu \\
       \addr Department of Psychiatry\\
	   University of Pittsburgh Medical Center\\
       University of Pittsburgh\\
       Pittsburgh, PA 15213, USA
       \AND
       \name Jihui Lee \email jil2043@med.cornell.edu \\
       \addr Division of Biostatistics, Department of Population Health Sciences\\
       Weill Cornell Medicine\\
       New York, NY 10065, USA}

\editor{}

\maketitle

\begin{abstract}
We consider the problem of interpretable network representation learning for samples of network-valued data. We propose the Principal Component Analysis for Networks (PCAN) algorithm to identify statistically meaningful low-dimensional representations of a network sample via subgraph count statistics. The PCAN procedure provides an interpretable framework for which one can readily visualize, explore, and formulate predictive models for network samples. We furthermore introduce a fast sampling-based algorithm, sPCAN, which is significantly more computationally efficient than its counterpart, but still enjoys advantages of interpretability. We investigate the relationship between these two methods and analyze their large-sample properties under the common regime where the sample of networks is a collection of kernel-based random graphs. We show that under this regime, the embeddings of the sPCAN method enjoy a central limit theorem and moreover that the population level embeddings of PCAN and sPCAN are equivalent. We assess PCAN's ability to visualize, cluster, and classify observations in network samples arising in nature, including functional connectivity network samples and dynamic networks describing the political co-voting habits of the U.S. Senate. Our analyses reveal that our proposed algorithm provides informative and discriminatory features describing the networks in each sample. The PCAN and sPCAN methods build on the current literature of network representation learning and set the stage for a new line of research in interpretable learning on network-valued data. Publicly available software for the PCAN and sPCAN methods are available at \url{https://www.github.com/jihuilee/PCAN}.
\end{abstract}

\begin{keywords}
Feature Engineering, Network Clustering, Network Embedding, Network Representation Learning, Unsupervised Learning
\end{keywords}

\section{Introduction}\label{sec:intro}

In the last decade, network representation learning (NRL) has become a common and important machine learning task for network-valued data. The goal of network representation learning is to identify low-dimensional representations of an observed network's vertices while preserving various aspects of the original graph like its topology, vertex attributes, or community structure \citep{hamilton2017representation, wang2017community, zhang2018network}. For a network with $n$ vertices, NRL seeks $p$-dimensional feature vectors with $p < n$ that capture representative structure of the original vertices and edges. These low-dimensional representations, called embeddings, are regularly used for routine machine learning tasks like clustering, classification, and prediction. NRL has provided important insights about relational systems in a variety of applications, including the study of social interactions \citep{perozzi2014deepwalk}, structural and functional connectivity of the brain \citep{rosenthal2018mapping, wilson2021analysis}, and information dissemination on the internet \citep{grover2016node2vec}.

Traditionally, NRL has amounted to manually describing summaries of networks with a collection of user-selected network properties, like structural importance or subgraph counts \citep{gallagher2010leveraging, henderson2011s}. Manually selected network properties provide interpretable features of a network, and have been well-utilized in applications of network science (e.g., \citet{ali2016comparison}). In the past few years, an expansive and growing suite of NRL methods have been introduced for static networks, and they have shown great promise for the exploration and analysis of complex network data -- see \citet{goyal2017graph} and \citet{hamilton2017representation} for recent reviews of network representation learning techniques. 

In many applications a sample of networks, rather than a single instance, is collected, and observations within that sample provide complementary information about the complex system under scrutiny. In the study of connectomics, for example, connectivity networks are collected across various modes - across time, individuals, or cognitive task \citep{betzel2016multi, bassett2011dynamic, muldoon2016network, wilson2020hierarchical}. Multi-transit transportation systems involve traffic measurements on the same geographic area across many forms of transportation \citep{cardillo2013emergence, strano2015multiplex}. An individual's social influence depends on their interactions across a collection of social media platforms rather than just a single platform \citep{OSELIO2018679}. In applications like these, relational data are collected as a sample of networks $\mathbf{G}_N := \{G_1, \ldots, G_N\}$, where observation $i$ is represented by the network $G_i = (V_i, E_i)$ with vertex set $V_i$ and edge set $E_i$. Despite the increasing prominence of network samples arising in nature, the tools available for analyzing network samples remain limited.

In this paper, we consider the problem of network representation learning for network samples. We develop a decomposition technique, known as Principal Component Analysis for Networks (PCAN), which identifies a statistically meaningful low-dimensional representation of a network sample $\mathbf{G}_N$. The embeddings identified by the PCAN algorithm are a collection of mutually orthogonal vectors describing the directions of most variability in the subgraph counts describing $\mathbf{G}_N$. The PCAN algorithm is an unsupervised method that identifies informative features that can be used to readily visualize, cluster, and train supervised learning algorithms on observations in a network sample. Despite its utility, the PCAN algorithm is limited by its computational speed due to the complexity of computing subgraph counts in large networks. To address this limitation, we develop a fast sampling-based PCAN procedure, sPCAN, that is significantly more efficient than the PCAN algorithm and still enjoys the same advantages of interpretability. We investigate the utility of the PCAN and sPCAN methods on two real applications and find that these strategies provide meaningful embeddings for the network samples under investigation. We also analyze the statistical properties of the PCAN and sPCAN methods through an analysis of the embeddings identified by the algorithms under the family of kernel-based random graph model. We provide a central limit theorem for the embeddings and show that the population embeddings of the two methods are equivalent under this regime.

The remainder of the paper is organized as follows. We first describe related work in Section \ref{sec:related}. In Section \ref{sec:method} we describe how to characterize a network sample with principal components and introduce the PCAN algorithm. We describe the sampling-based sPCAN method in Section \ref{sec:sPCAN}. In Section \ref{sec:statproperties}, we explore the relationship between the PCAN and sPCAN algorithms and investigate the large-sample properties of the sPCAN algorithm when the observed sample consists of networks that are members of the kernel-based random graph model. We apply PCAN and sPCAN to two diverse applications -- a dynamic network of political co-voting data of the U.S. Senate, as well as a sample of functional connectivity matrices comparing healthy controls and patients with schizophrenia in Section \ref{sec:application}. We conclude with a discussion of open problems and areas of future work in Section \ref{sec:conclusion}.

\subsection{Related Work}\label{sec:related}

There has been a surge of interest in network representation learning for static network data in recent years. Generally speaking, NRL methods fall into one of two categories: generative or descriptive. Descriptive methods like LINE \citep{tang2015line}, DeepWalk \citep{perozzi2014deepwalk}, and node2vec \citep{grover2016node2vec} are designed to identify features by scanning local structures of the network like neighborhoods and shortest paths. Descriptive NRL methods have only recently been extended to the setting of network samples and have largely relied on adaptations of static methods \citep{wilson2021analysis}. Descriptive methods have caught a lot of attention in the last few years for at least two major reasons: (i) they are fast -- algorithms are designed to quickly embed large graphs (millions of nodes) in minutes, (ii) the features identified are highly predictive of other node-based attributes and motifs like community structure. In many situations, descriptive methods rely on the application of a neural network architecture and the identified features have been shown to be approximations of other well-understood unsupervised strategies like non-negative matrix factorization \citep{levy2014neural}. Generative methods, on the other hand, characterize lower-dimensional features of the network through a generative probabilistic model. The latent space model from \cite{hoff2002latent}, for example, is a common model-based embedding technique that embeds the observed network onto Euclidean space based on node positions after accounting for other network features and attributes in the model. Unlike typical descriptive methods, generative models like the latent space model benefit from interpretability of node embeddings but are comparatively computationally slow. PCAN and sPCAN are fast descriptive methods that enjoy interpretability only directly available to generative methods.

There are several other studies concerning representation learning for network samples related to our current work. Perhaps most closely associated with the PCAN and sPCAN algorithms are embedding algorithms designed for multilayered networks including multi-node2vec \citep{wilson2021analysis} and the multilayered embedding methods described in \citet{liu2017principled}. These strategies are descriptive methods that generate embeddings for networks with multiple layers, where network layers generally represent differing edge types among a similar set of vertices. Network samples are a special case of multilayered networks where layers are treated as independent observations. There is also some recent work focused on modeling samples of networks, including the random effects stochastic block model \citep{paul2018random} the multi-subject stochastic block model \citep{pavlovic2019multi}, the edge-based logistic model \citep{simpson2019mixed}, and the hierarchical latent space model \citep{wilson2020hierarchical}. Each of these strategies are generative embedding procedures for network samples that depend on the estimation of a probabilistic model describing the sample.

The PCAN method nonparametrically characterizes network samples via subgraph count statistics. Subgraph counts and related network summaries are commonly used in descriptive analysis to explore the properties of network-valued data; such analyses are well studied (see for example \citet{gallagher2010leveraging, henderson2011s, ali2016comparison}). Subgraph counts are also commonly treated as features in generative models for network-valued data. The family of exponential random graph models, for example, use subgraph counts as sufficient statistics for an observed network \citep{holland1981exponential}. Similar analyses in \citet{lovasz2012large} established the sufficiency of subgraph counts under several popular classes of generative models in the family of kernel-based random graphs, including the stochastic block model \citep{holland1983stochastic, wang1987stochastic}, the \erdos random graph model \citep{erdos1961evolution}, and the latent space model \citep{hoff2002latent}. The large-sample analysis of the PCAN method builds on the foundations established in \citet{lovasz2012large} as well the first order analyses of subgraph counts in \citet{maugis2020testing}. In our current work, we analyze use of subgraph counts as sufficient statistics under the kernel-based random graph model when the objective is to identify principal components of the network sample.

Principal component analysis (PCA) is a well-studied multivariate embedding technique for multivariate data with a wide array of applications, including topic modeling, micro-array analysis, and computer vision (for a recent tutorial, see \citet{shlens2014tutorial}). The objective of multivariate data embedding, sometimes referred to as multidimensional scaling, is to reduce the dimension of an $n \times p$ data matrix $\mathbf{X}$ whose rows are assumed to be $n$ independent observations from the population $\mathcal{X} \in \mathbf{R}^p$.  Applying PCA to a (tabular) matrix $\mathbf{X}$ results in two primary data objects of interest. The first is a collection of $p$ orthonormal vectors of length $p$, called the {principal components} (PCs), where the $j$th vector is the linear combination of the columns of $\mathbf{X}$ that explain the $j$th most variability in $\mathbf{X}$. The second output of PCA is a collection of $p$ scalars, known as the {loadings}, whose $j$th value quantifies the variability of $\mathbf{X}$ explained by the $j$th PC. The PCAN and sPCAN methods are generalizations of PCA developed for the analysis of network samples. Both methods benefit from the advantages of PCA for exploration, visualization, and comparative analysis of network samples. We note that there is a growing literature on the development and analysis of multivariate PCA for high-dimensional (tabular) data which may provide fruitful future directions in the analysis and use of PCA for network-valued data with special high-dimensional structure (see for example \citet{cai2013sparse, jung2009pca}).

\section{Describing Network Samples with Principal Components}\label{sec:method}

The goal of the PCAN algorithm is to provide a low-dimensional summary of a network sample $\mathbf{G}_N = \{G_1, \ldots, G_N\}$ that is descriptive, interpretable, and captures the connectivity patterns of the networks in the sample. With these criteria in mind, we propose identifying the principal components of a descriptive collection of subgraph densities characterizing the connectivity patterns of observations in $\mathbf{G}_N$. Below, we describe the PCAN algorithm, starting with a discussion of subgraph densities.

\subsection{Subgraph Densities of a Network Sample}

We begin by describing $\mathbf{G}_N$ using subgraph densities. Informally, a subgraph count is the number of occurrences of a particular subgraph within the network $G$. A subgraph density is the ratio of the subgraph count to the maximum possible count of the subgraph in graphs with the same number of nodes as $G$. To make this precise, we establish some notations and definitions.

\begin{definition}\label{def_subgraph} [Subgraph] A graph $F = (V_F, E_F)$ is said to be an induced subgraph of $G = (V_G, E_G)$ (written $F \subset G$) if $V_F \subseteq V$ and the edge set $E_F$ is constructed $E_F = \{(u,v): u,v \in V_F ~~\text{and}~~ (u,v) \in E\}$.
\end{definition}

We now need to define a notion of equivalence of two graphs $F$ and $F'$. For this, we consider equivalence up to node label permutations, defined below.

\begin{definition}\label{def_equiv} [Graph equivalence] We say that two unweighted graphs $F = (V_F, E_F)$ and $F' = (V_{F'}, E_{F'})$ are equivalent (written $F \equiv F'$) if there exists a bijective map $\phi: V_F \rightarrow V_{F'}$ such that $(u,v) \in E_{F}$ if and only if $(\phi(u), \phi(v)) \in E_{F'}$.
\end{definition}

We can now formally define a subgraph count and subgraph density using Definitions \ref{def_subgraph} and \ref{def_equiv}.

\begin{definition}\label{def_subgraph_count} [Subgraph count, subgraph density] A subgraph count of the configuration $F$ in the graph $G = (V_G, E_G)$, notated by $S_F(G)$, is the number of subgraphs in $G$ that are equivalent to $F$. The subgraph density of $F$ in $G$, notated by $\widetilde{S}_F(G)$, is the ratio of $S_F(G)$ to the maximum possible subgraph count of $F$ in graphs with the same number of nodes as $G$. Formally, let $M_F(G)$ denote the maximum possible subgraph count of $F$ on the nodes $V_G$. Then $S_F(G)$ and $\widetilde{S}_F(G)$ are defined as follows:
\vskip -1pc
\begin{equation}
	S_F(G) = \#\{F' \subset G: F' \equiv F\}
\end{equation} \vskip -1pc
\begin{equation} \label{subgraph_density}
	\widetilde{S}_F(G) = \dfrac{S_F(G)}{M_F(G)}
\end{equation}
\end{definition}

Subgraph densities are relatively simple to calculate and have been found to provide interpretable features about the interactions of the system under study \citep{kolaczyk2014statistical}. For example, the density of triangles in a social network describes the clustering tendency of the individuals in the network; this value has been shown to be an important feature of international trade relations and political co-voting tendencies \citep{oatley2013political, wilson2017stochastic, lee2020varying}. In neuroimaging applications, the number of $k$-star configurations in functional connectivity networks of typically developing adults is significantly higher than expected at random \citep{stillman2017statistical, stillman2019consistent}. Examples of subgraph configurations are given in Figure \ref{fig:examples}. 

\begin{figure}[htbp!]
    \centering
    \begin{subfigure}[t]{0.19\textwidth} 
        \centering
\begin{tikzpicture}[scale = 0.7, shorten >=1pt,-]
  \tikzstyle{vertex}=[circle,fill=black!25,minimum size=17pt,inner sep=0pt]
  \foreach \name/\angle/\text in {G-1/180/i, G-2/360/j}
    \node[vertex,xshift=.5cm,yshift=6cm] (\name) at (\angle:1cm) {$\text$};
  \foreach \from/\to in {1/2}
    \draw (G-\from) -- (G-\to);
\end{tikzpicture}  \caption{Edge (One-star)}
    \end{subfigure} \hfill
    \begin{subfigure}[t]{0.19\textwidth} 
        \centering
\begin{tikzpicture}[scale = 0.7, shorten >=1pt,-]
  \tikzstyle{vertex}=[circle,fill=black!25,minimum size=17pt,inner sep=0pt]
  \foreach \name/\angle/\text in {G-1/180/i, G-2/90/j, G-3/0/k}
    \node[vertex,xshift=6cm,yshift=.5cm] (\name) at (\angle:1cm) {$\text$};
  \foreach \from/\to in {1/2, 2/3}
    \draw (G-\from) -- (G-\to);
\end{tikzpicture}  \caption{Two-star}
    \end{subfigure} \hfill
    \begin{subfigure}[t]{0.19\textwidth} 
        \centering
\begin{tikzpicture}[scale = 0.7, shorten >=1pt,-]
  \tikzstyle{vertex}=[circle,fill=black!25,minimum size=17pt,inner sep=0pt]
  \foreach \name/\angle/\text in {G-1/180/i, G-2/90/j, G-3/0/k}
    \node[vertex,xshift=6cm,yshift=.5cm] (\name) at (\angle:1cm) {$\text$};
  \foreach \from/\to in {1/2, 2/3, 3/1}
    \draw (G-\from) -- (G-\to);
\end{tikzpicture}  \caption{Triangle}
    \end{subfigure} \hfill        
    \begin{subfigure}[t]{0.19\textwidth} 
        \centering
\begin{tikzpicture}[scale = 0.7, shorten >=1pt,-]
  \tikzstyle{vertex}=[circle,fill=black!25,minimum size=17pt,inner sep=0pt]
  \foreach \name/\angle/\text in {G-1/180/i, G-2/90/j, G-3/0/k, G-4/270/l}
    \node[vertex,xshift=6cm,yshift=.5cm] (\name) at (\angle:1cm) {$\text$};
  \foreach \from/\to in {1/2, 2/3, 3/4, 4/1}
    \draw (G-\from) -- (G-\to);
\end{tikzpicture}  \caption{Cycle(4)}
    \end{subfigure} \hfill
    \begin{subfigure}[t]{0.19\textwidth} 
        \centering
\begin{tikzpicture}[scale = 0.7, shorten >=1pt,-]
  \tikzstyle{vertex}=[circle,fill=black!25,minimum size=17pt,inner sep=0pt]
  \foreach \name/\angle/\text in {G-1/165/i, G-2/90/j, G-3/15/k, G-4/300/l, G-5/240/m}
    \node[vertex,xshift=6cm,yshift=.5cm] (\name) at (\angle:1cm) {$\text$};
  \foreach \from/\to in {1/2, 2/3, 3/4, 4/5, 5/1}
    \draw (G-\from) -- (G-\to);
\end{tikzpicture}  \caption{Cycle(5)}
    \end{subfigure} \hfill
    \caption{Examples of subgraph configurations used to summarize a graph sample. \label{fig:examples}}
\end{figure}
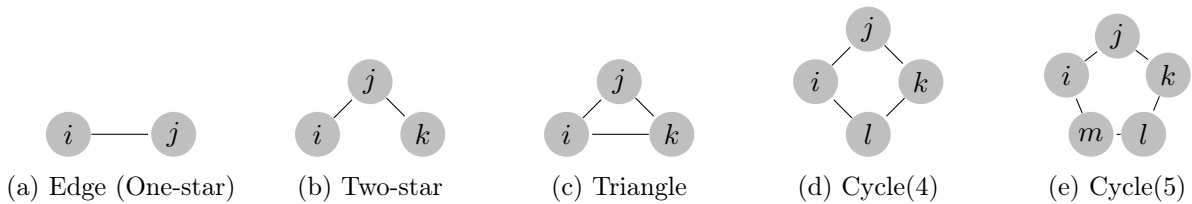

In addition to their descriptive power, subgraph densities are often used as features in generative probabilistic models characterizing an observed network. Subgraph densities are sufficient statistics for the exponential random graph model, for example, a model that is widely-used to describe social interactions and behavior \citep{holland1981exponential, szekely2016childhood}. Subgraph densities are also sufficient statistics for a large class of kernel-based random graphs \citep{lovasz2012large} and benefit from several useful properties that build the foundation of the large-sample analysis of the PCAN algorithm. We discuss these properties and our large-sample results in Section \ref{sec:statproperties}.

\subsection{Principal Component Analysis of Network Samples}\label{sec:PCAN}

To characterize the connectivity patterns of $\mathbf{G}_N$, we first calculate the subgraph density for $p$ representative subgraph configurations $\mathbf{F} = \{F_1, \ldots, F_p\}$. The measurement for $G_i$ is stored as the $p$-dimensional vector of subgraph densities $\widetilde{\mathbf{S}}_{\mathbf{F}}(G_i) = (\widetilde{S}_{F_1}(G_i), \ldots, \widetilde{S}_{F_p}(G_i))^T$. If we view these densities as sufficient statistics for each $G_i$, then one can readily characterize $\mathbf{G}_N$ using first-order summaries of these vectors. \cite{maugis2020testing} analyzed the properties of the mean subgraph density and developed two-sample hypothesis tests for comparing the means of two network samples. There are, however, several limitations to focusing our analysis on first-order summaries of $\widetilde{\mathbf{S}}_{\mathbf{F}}$. First, subgraph counts (and hence densities) are often highly correlated. Not only does this discourage stability of the variance of the sample mean, but it also confiscates the marginal effect of each subgraph configuration on the network sample. Second, when the number of subgraph configurations $p$ is large, it becomes challenging to visualize and explore the relationships of these subgraph densities with observations in the sample. Pairwise plots of the subgraph densities would require on the order of $p^2$ plots, which is infeasible even for modest $p$.

To aid the exploratory analysis of $\mathbf{G}_N$, we propose computing the principal components of subgraph densities. By calculating the principal components, we can identify interpretable and uncorrelated features that describe the variability of subgraph configurations in the network sample. Define the subgraph density matrix as the $p \times N$ matrix

\begin{equation} \label{eq:stacked}\mathbf{S} = [\widetilde{\mathbf{S}}_{\mathbf{F}}(G_1), \widetilde{\mathbf{S}}_{\mathbf{F}}(G_2), \cdots, \widetilde{\mathbf{S}}_{\mathbf{F}}(G_N)], \end{equation}

\noindent whose columns are the subgraph density vectors for each network in $\mathbf{G}_N$. Without loss of generality, suppose that the density vectors are mean centered, namely the rows of $\mathbf{S}$ each have mean 0. We note that in practice we also standardize the rows of $\mathbf{S}$ to have unit standard deviation to ensure the scale of each network density is comparable. We can then define the $p \times p$ empirical covariance matrix of $\mathbf{S}$, given by
\vskip -1pc
\begin{equation} \label{eq:sigma_pop} \Sigma_{\mathbf{S}}: = \frac{1}{N}\mathbf{S} \mathbf{S}^T. \end{equation}

\noindent The $(j,k)$th entry of $\Sigma_{\mathbf{S}}$ is the empirical covariance of $\widetilde{\mathbf{S}}_{F_j}$ and $\widetilde{\mathbf{S}}_{F_k}$ calculated over the $N$ network observations in $\mathbf{G}_N$. The PCAN algorithm identifies the $p$-dimensional principal components of the subgraph density matrix $\mathbf{S}$ from (\ref{eq:stacked}). This is achieved by first identifying the eigenvectors and eigenvalues of $\Sigma_{\mathbf{S}}$ and then projecting $\mathbf{S}$ onto the space of the leading eigenvectors of $\Sigma_{\mathbf{S}}$. This is described in the pseudo-code below.

\begin{framed}
	\begin{center}{\bf The PCAN Algorithm} \end{center}
		{\small
	\begin{itemize}
		\item {\bf Input}: 
		\begin{itemize}
			\item A network sample $\mathbf{G}_N = \{G_1, \ldots, G_N\}$ 
			\item A collection of subgraph configurations $\mathbf{F} = \{F_1, \ldots, F_p\}$
			\item $r$: the number of desired principal components 
		\end{itemize}
		\item Calculate the mean-centered subgraph density matrix $\mathbf{S}$ and the empirical covariance matrix $\Sigma_{\mathbf{S}}$, defined in (\ref{eq:stacked}) and (\ref{eq:sigma_pop}), respectively.
		\item Calculate the $r$ leading sample eigenvectors $\widehat{\mathbf{v}}_1, \ldots, \widehat{\mathbf{v}}_r$ and the $r$ leading sample eigenvalues $\widehat{\lambda}_1, \ldots, \widehat{\lambda}_r$ of $\Sigma_{\mathbf{S}}$. Calculate the $r$ top principal components $\{\widehat{\mathbf{u}}_1, \ldots \widehat{\mathbf{u}}_r\}$ defined by
		$$\widehat{\mathbf{u}}_\ell = \mathbf{S}^T \widehat{\mathbf{v}}_\ell, \hskip 2pc \ell = 1,\ldots, r.$$
		\item {\bf Return}: $\widehat{\mathbf{u}}_1, \ldots \widehat{\mathbf{u}}_r$ and $\widehat{\lambda}_1, \ldots, \widehat{\lambda}_r$
	\end{itemize}}
\end{framed}

\subsection{Interpretation of PCAN} \label{sec:interpretation_PCAN}
The embeddings $\widehat{\mathbf{u}}_1, \ldots \widehat{\mathbf{u}}_r$ and their associated scores $\widehat{\lambda}_1, \ldots, \widehat{\lambda}_r$ returned by the PCAN algorithm have three straightforward interpretations, which provide an easy to understand description of the network sample $\mathbf{G}_N$. We state these properties without proof as they are well-understood (see Chapter 10 of \citet{james2013introduction} for example). \\

\noindent \underline{\bf Properties of PCAN}
\begin{enumerate}
	\item The vector $\widehat{\mathbf{v}}_1$ is the direction in the rowspace of $\mathbf{S}$ along which the observations vary the most. Furthermore, given $\widehat{\mathbf{v}}_{(\ell)} = \{\widehat{\mathbf{v}}_1, \ldots, \widehat{\mathbf{v}}_\ell\}$, $\widehat{\mathbf{v}}_{\ell+1}$ is the direction, which is orthogonal to all vectors in $\widehat{\mathbf{v}}_{(\ell)}$, along which the observations vary the most.
	\item The percentage of variability in $\mathbf{S}$ described by the $r$ top principal components is given by $\sum_{\ell = 1}^r \widehat{\lambda}_\ell$ / $\sum_{\ell = 1}^p \widehat{\lambda}_\ell$.
	\item The first $r$ principal components span the $r$-dimensional hyperplane that are closest in Euclidean distance to the $N$ columns of $\mathbf{S}$. In particular, let $\widehat{u}_{i,\ell}$ and $\widehat{v}_{j,\ell}$ denote the $i$th and $j$th entry of $\widehat{\mathbf{u}}_\ell$ and $\widehat{\mathbf{v}}_\ell$, respectively. Then the approximation for $\widetilde{S}_{F_j}(G_i)$ that is closest in Euclidean distance is given by:
	\vskip -1pc
	\begin{equation}\label{eq:approx} \widetilde{S}_{F_j}(G_i) \approx \sum_{\ell = 1}^r \widehat{u}_{i,\ell} \widehat{v}_{j,\ell}. \end{equation}
\end{enumerate}

The three properties above suggest several important strategies for how to analyze a network sample using the PCAN algorithm. First, properties 1 and 2 reveal that PCAN embeddings describe the variability of the subgraph densities of a network sample, and moreover that these embeddings are uncorrelated. These consequences are particularly useful for both assessing the underlying dimension of the network sample, as well as visualizing the sample. The principal component loadings for each subgraph configuration enable the exploration of the topological characteristics of the sample that contribute most to the variability in the sample. One can use scree plots to plot the variability explained by each principal component and biplots to investigate pairwise relationships of the loadings of the principal components. Such visualizations are valuable low-dimensional exploratory tools that can be used for a number of applications, including the detection of outlying networks in the sample, the clustering of networks in the sample, as well as comparison of two network samples. Another important consequence of properties 1 and 2 is the fact that the ordering of the embeddings has meaning. The first embedding explains the most variability in subgraph densities, the second explains the next most variability, and so on. As a result, the comparison of two network samples can be done through a direct comparison of the samples' principal components. This desirable property is not available for state-of-the-art methods like node2vec and its variants. 

Property 3 above suggests a strategy for approximating the subgraph densities of observations within the sample. Importantly, (\ref{eq:approx}) provides a direct relationship between the embeddings of PCAN to well-understood subgraph densities in the network sample. This is in stark contrast to contemporary embedding strategies, which are unable to capture such features of the network \citep{seshadhri2020impossibility}. Furthermore, one can use (\ref{eq:approx}) to identify outlying networks in the sample whose subgraph densities are ``far'' from what is expected in the remainder of the sample \underline{without} needing the original network sample itself. This is particularly useful in applications of network compression where embeddings can be given but the network observations themselves cannot for privacy reasons. We do not consider these applications further in this manuscript, but plan to pursue this line of work in future research.

\section{The Sampling-based PCAN Procedure for Large Networks}\label{sec:sPCAN}

The PCAN procedure provides an interpretable framework for which we can readily visualize and explore network samples. Unfortunately, when observations in $\mathbf{G}_N$ are particularly large, the PCAN algorithm is limited by computational speed due to the computational complexity of calculating subgraph counts in a large graph. Calculating the number of triangles in a graph of size $n$, for example, has a computational complexity of $O(n^3)$ making it infeasible for networks of modest size (thousands of vertices or more). In light of this challenge, we propose a fast sampling-based principal component analysis procedure, which we call the sPCAN method, that enjoys the same advantages of interpretability as the PCAN method.

Like PCAN, the sPCAN procedure also takes as input a network sample $\mathbf{G}_N = \{G_1, \ldots, G_N\}$, and collection of subgraph configurations $\mathbf{F} = \{F_1, \ldots, F_p\}$ for which subgraph densities will be calculated. Rather than calculating subgraph densities on each full network in $\mathbf{G}_N$, however, the sPCAN method first samples induced subnetworks from each $G_i \in \mathbf{G}_N$ and subsequently calculates subgraph densities from these smaller induced subnetworks. Subgraph density vectors are averaged over induced subnetworks for each observation in $\mathbf{G}_N$ and stacked side by side to generate a mean subgraph density matrix $\mathbf{M}$. Multivariate principal component analysis is then applied directly to $\mathbf{M}$. The sPCAN method requires two tuning parameters $\tau \geq 2$ and $K \geq 1$ that dictate the size of subnetwork samples. The sPCAN algorithm consists of three primary stages: 1) network sampling, 2) calculating the mean subgraph density matrix, and 3) applying multivariate principal component analysis, which we describe in detail below. 

\subsection{Step 1: Partition Networks}
Given a pre-specified $K$ and $\tau$, the sPCAN algorithm first randomly partitions each graph $G_i = (V_i, E_i) \in \mathbf{G}_N$ into $K$ induced subnetworks, each of which contains at least $\tau$ vertices. This is achieved by randomly partitioning the vertex set $V_i$ into $K$ disjoint subsets $\{V_{i,1}, \ldots, V_{i,K}\}$ in the following manner. A class label $c_u \in \{1, \ldots, K\}$ is assigned for each vertex $u \in V_i$ through random and independent draws from a multinomial distribution with $K$ classes and class probability $p_k \equiv 1/K$ for each $k = 1, \ldots, K$ with the constraint that the minimum class size is greater than equal to $\tau$. We then set $V_{i,k}$ to be the collection of vertices with label $k$, namely, $V_{i,k} = \{u \in V_i: c_u = k\}$. The edge set $E_{i,k}$ is the set of all edges that connect vertices in $V_{i,k}$, namely, $E_{i,k} = \{\{u,v\}: u, v \in V_{i,k} ~~\text{and}~~ \{u,v\} \in E_i\}$. The induced subgraph $G_{i,k}$ is defined as $G_{i,k} = (V_{i,k}, E_{i,k})$. This step outputs a collection of $N*K$ induced subgraphs $\{G_{i,k}: i = 1, \ldots, N; k = 1, \ldots, K\}$. \\

\noindent \underline{\bf The choice of $K$ and $\tau$}

The parameters $K$ and $\tau$ are tuning parameters and both depend on the size of smallest network in the sample $\mathbf{G_N}$. Additionally, $\tau$ depends on the specified collection of subgraph configurations $\mathbf{F}$. Suppose that the number of vertices in the smallest network in $\mathbf{G}_N$ is $n_{min}$. For a specified $\mathbf{F}$ and input graph sequence $\mathbf{G}_N$, $K$ and $\tau$ must satisfy the following constraint
\vskip -1pc
\begin{equation}\label{eq:constraint} 2\max_j|F_j| \leq \tau \leq \dfrac{n_{min}}{K} \end{equation}

The bounds given by (\ref{eq:constraint}) provide a strategy for how to choose $K$ and $\tau$ in practice. First, subgraph configurations must be chosen so that $\max_j|F_j| \leq n_{min} / 2$ in order to ensure that subgraph densities can be calculated for each induced subnetwork. As long as this condition is met, we can set $\tau$ to be as small as possible, namely $\tau = 2 \max_j |F_j|$. Subsequently, we choose $K$ to be as large as possible, which according to (\ref{eq:constraint}) is $K = \floor{n_{min} / \tau}$. By default, we set $\tau$ and $K$ in this manner. We note that choosing $K$ to be as large as possible is recommended because of the role that $K$ plays in the asymptotic equivalence of sPCAN with PCAN. Specifically, equivalence of the two methods relies on large $K$, which in typical examples is reasonable since one only needs to use sPCAN as an alternative to PCAN when graphs in the sample are large. We investigate this further in Section \ref{sec:statproperties}.

\subsection{Step 2: Calculate the mean subgraph density matrix}

Given the collection of induced subgraphs from Step 1, the sPCAN algorithm now calculates densities of the subgraph configurations specified in $\mathbf{F}$. The algorithm proceeds by calculating a mean subgraph density matrix, $\mathbf{M}$, which summarizes the variation in the densities of the subgraph configurations. In particular, calculation of the matrix $\mathbf{M}$ is carried out in three steps:

 \begin{enumerate}
 	\item[a.] Calculate subgraph density vectors $\mathbf{S}_{i, k} = (\widetilde{S}_{F_1}(G_{i,k}), \ldots, \widetilde{S}_{F_p}(G_{i,k}))^T$ for each $G_{i,k}$, $k = 1, \ldots, K$.
 	\item[b.] For each observation $G_i \in \mathbf{G}_N$, average over the subnetworks to obtain $p \times 1$ mean subgraph density vectors:
 	$$\overline{\mathbf{S}}_i = \dfrac{1}{K} \sum_{k = 1}^K \mathbf{S}_{i,k}, \hskip 2pc i = 1, \ldots, N$$
 	\item[c.] Construct the $p \times N$ mean subgraph density matrix $\mathbf{M}$:
 	\begin{equation}\label{eq:M} \mathbf{M}: = [\overline{\mathbf{S}}_1, \overline{\mathbf{S}}_2, \ldots, \overline{\mathbf{S}}_{N}]. \end{equation}
  	\end{enumerate}

\subsection{Step 3: Apply multivariate PCA to $\mathbf{M}$}
In the third and final step of the PCAN algorithm, multivariate principal component analysis is applied to the subgraph density matrix $\mathbf{M}$ from step 2. This requires 2 steps, as provided below.
\begin{enumerate}
		\item Without loss of generality, assume that the rows of $\mathbf{M}$ have been mean-centered. Calculate the empirical covariance matrix of $\mathbf{M}$: \vskip -1pc		
\begin{equation}\label{sigma_M}{\Sigma}_{\mathbf{M}} := \dfrac{1}{N}{\mathbf{M}}{\mathbf{M}}^T. \end{equation}
		\item Calculate the $r$ leading eigenvectors $\widehat{\mathbf{w}}_1, \ldots, \widehat{\mathbf{w}}_r$, and the $r$ leading eigenvalues $\widehat{\gamma}_1, \ldots, \widehat{\gamma}_r$, and eigenvectors of $\Sigma_{\mathbf{M}}$. Calculate the $r$ top principal components $\widehat{\mathbf{y}}_1, \ldots, \widehat{\mathbf{y}}_r$ defined by
		$$\widehat{\mathbf{y}}_\ell = \mathbf{M}^T\widehat{\mathbf{w}}_\ell, \hskip 2pc \ell = 1,\ldots, r.$$
	\end{enumerate}

To summarize, the sPCAN procedure first characterizes the sample $\mathbf{G}_N$ with the densities of subgraph configurations of a collection of induced subnetworks from each $G_i \in \mathbf{G}_N$. The algorithm then assesses what linear combination of subgraph configurations describes the most variability in the sample partition using multivariate PCA. The pseudo-code for the sPCAN method is provided below. 

\begin{framed}
	\begin{center}{\bf The sPCAN Algorithm} \end{center}
		{\small
	\begin{itemize}
		\item {\bf Input}: 
		\begin{itemize}
			\item A network sample $\mathbf{G}_N = \{G_1, \ldots, G_N\}$ 
			\item A collection of subgraph configurations $\mathbf{F} = \{F_1, \ldots, F_p\}$
			\item $r$: the number of desired principal components 
			\item $K$: the number of subnetworks to sample from each observation
			\item $\tau$: the minimum size of a sampled subnetwork
		\end{itemize}
		\item {\bf For} each graph $G_i = (V_i, E_i) \in \mathbf{G}_N$:
		\begin{itemize}
			\item[] {\bf For} $u = 1, \ldots, |V_i|$:
			\begin{itemize}
				\item[] Independently draw the vertex label $c_u \sim \text{Multinomial}(1, (p_1, ..., p_K))$ where $p_k = 1/K$, $k = 1, \ldots, K$, ensuring that $\#\{c_u = k\} \geq \tau$ for each $k$.
			\end{itemize}
			\item[] Set $V_{i,k} = \{u \in V_i: c_u = k\}$ and $E_{i,k} = \{\{u,v\}: u,v \in V_{i,k}~~\text{and}~~\{u,v\} \in E_i\}$
		\end{itemize}
		\item Calculate the mean-centered mean subgraph density matrix $\mathbf{M}$ and the empirical covariance matrix $\Sigma_{\mathbf{M}}$, defined in (\ref{eq:M}) and (\ref{sigma_M}), respectively.
		\item Calculate the $r$ leading sample eigenvectors $\widehat{\mathbf{w}}_1, \ldots, \widehat{\mathbf{w}}_r$ and the $r$ leading sample eigenvalues $\widehat{\gamma}_1, \ldots, \widehat{\gamma}_r$ of $\Sigma_{\mathbf{M}}$. Calculate the $r$ top principal components $\{\widehat{\mathbf{y}}_1, \ldots \widehat{\mathbf{y}}_r\}$ defined by
		$$\widehat{\mathbf{y}}_\ell = \mathbf{M}^T\widehat{\mathbf{w}}_\ell, \hskip 2pc \ell = 1,\ldots, r.$$
		\item {\bf Return}: $\widehat{\mathbf{y}}_1, \ldots \widehat{\mathbf{y}}_r$ and $\widehat{\gamma}_1, \ldots, \widehat{\gamma}_r$
	\end{itemize}}
\end{framed}

\subsection{Interpretation of sPCAN}

The sPCAN method gains substantial computational speed over PCAN, making it suitable for network samples with large network observations. For example, calculating the number of triangles for a graph of size $n$ in sPCAN can be distributed across $K$ processes with time $O((n/K)^3)$ compared to $O(n^3)$ in the PCAN algorithm. Thus when the size of average network in the sample is close to the size of the smallest network, $n_{min}$, we can set $K = n_{min} / \tau$ to calculate the number of triangles for each network in the sample in roughly constant time.

At first glance, it may seem that the sPCAN algorithm loses the advantages of interpretability available to the PCAN algorithm outlined in Section \ref{sec:interpretation_PCAN}. On the contrary, it turns out that when observations from $\mathbf{G}_N$ are assumed to be drawn independently from a the family of kernel-based random graph models -- a common family of random graph models -- the principal components identified by sPCAN are equal in expectation to those identified by PCAN (see Theorem \ref{thm:equiv} in Section \ref{sec:statproperties}). As a consequence, when samples are drawn from the kernel-based random graph model, the principal components identified from the sPCAN method are approximately equal to those identified by the PCAN method. The sPCAN method also benefits from a central limit theorem for both $\mathbf{M}$ and the embeddings in the large sample, large partition $N, K \rightarrow \infty$ regime when $\mathbf{G}_N$ is drawn from a kernel-based random graph family. These properties are discussed in detail in Section \ref{sec:statproperties}.

\section{Statistical Properties of PCAN and sPCAN}\label{sec:statproperties}

In this section we investigate the statistical properties of the PCAN and sPCAN methods. We first describe the relationship between the two methods, and then describe the large sample ($N \rightarrow \infty$) asymptotics of the principal components output by the sPCAN algorithm described in Section \ref{sec:sPCAN}. Our analysis assumes that the observed sample of graphs are drawn from a family of kernel-based random graphs (KRGs). The family of KRGs is a widely studied class of random graph models that contain many popular statistical network models for unweighted graphs \citep{lovasz2012large}, including the \erdos random graph model \citep{erdos1961evolution}, the stochastic block model \citep{holland1983stochastic, wang1987stochastic}, the latent space model \citep{hoff2002latent}, the beta model \citep{chatterjee2011random, britton2006generating}, as well as the random dot-product graph model \citep{young2007random, athreya2017statistical}.   

To generate a random graph from the KRG family, one first randomly assigns a latent variable on the unit interval to each vertex independently. Pairs of vertices are then connected with an edge independently conditional on the latent variables via a continuous function $f: [0,1]^2 \rightarrow [0,1]$. A formal definition is given below

\begin{definition}[Kernel-based random graph]\label{def:kbrg}
	Let $[n] = \{1, \ldots, n\}$ and let $f$ be a continuous function where $f: [0,1]^2 \rightarrow [0,1]$. Define $\mathcal{G}(n, f)$ as the probability distribution over all graphs with vertex set $[n]$ and edge set $E$ such that
	\begin{itemize}
		\item[(i)] each node $i \in [n]$ is independently and randomly assigned a feature $x_i \sim U(0,1)$
		\item[(ii)] edges form independently conditional on $\{x_i\}_{i \in [n]}$ with probability
		$$P(\{i,j\} \in E | x_i, x_j) = f(x_i, x_j).$$
	\end{itemize}

We say that $G$ is a kernel-based random graph with kernel $f$ if $G$ satisfies (i) and (ii). Throughout, we will write $G \sim \mathcal{G}(n, f)$ to mean that $G$ is a random graph with $n$ vertices from a kernel-based random graph model with kernel $f$. 
 
\end{definition}

Our main results show that under the KRG family, (1) the output embeddings of the sPCAN and PCAN methods are asymptotically equal in expectation [Theorem \ref{thm:equiv}], (2) the mean subgraph density matrix $\mathbf{M}$ is asymptotically normal as $K \rightarrow \infty$ [Proposition \ref{lem_CLT1}], and as a result (3) the sample eigenvalues and eigenvectors obtained from the sPCAN algorithm are asymptotically normally distributed under minimal assumptions as $K, N \rightarrow \infty$ [Theorem \ref{thm:main}]. 
  
To begin we state an important property of KRGs, which is useful to our analysis of the PCAN and sPCAN algorithms that shows that an induced subgraph of a KRG with kernel $f$ is also a KRG with kernel $f$. We make this precise in the following lemma.

\begin{lemma}\label{lem_reduce_property}
	Suppose that $G_i$ is a kernel-based random graph with kernel $f$, namely $G_i \sim \mathcal{G}(n_i, f)$. Consider an induced subgraph $G_{i,k}$ from $G_i$ that contains at least 2 vertices. Let $n_{i,k} = |V_{i,k}|$ denote the size of the subgraph $G_{i,k}$. Then $G_{i,k} \sim \mathcal{G}(n_{i,k}, f)$. It follows for a partition of induced subgraphs $\{G_{i,1}, \ldots, G_{i,K}\}$ from $G_i$ for which $n_{i,k} \geq 2$, that $G_{i,k} \sim \mathcal{G}(n_{i,k}, f)$ for each $k = 1, \ldots, K$.
\end{lemma}

\begin{proof}
	We only need to show that $G_{i,k}$ satisfies (i) and (ii) of Definition \ref{def:kbrg}. Since $G_{i,k}$ is an induced subgraph of $G_i$ it follows that for all $u \in V_{i,k}$ there is an associated independent and randomly assigned latent feature $x_u \in U(0,1)$; thus (i) is satisfied. Since $E_{i,k} \subseteq E_i$ and $E_{i,k} \cap E_{i,k'} = \emptyset$ for $k \neq k'$, and every vertex of $G_i$ belongs to exactly one subgraph, it follows that for all pairs $\{u,v\}$ with $u, v \in V_{i,k}$,
	 $$P(\{u,v\} \in E_{i,k} | x_u, x_v) = P(\{u,v\} \in E_i | x_u, x_v) = f(x_u, x_v).$$ The above holds analagously for all $k = 1, \ldots, K$. Therefore (ii) is satisfied and we conclude that $G_{i,k} \sim \mathcal{G}(n_{i,k}, f)$ for all $k = 1, \ldots, K$.
\end{proof}

Lemma \ref{lem_reduce_property} suggests that the subgraph densities of induced subnetworks of a network is sufficient for estimating the kernel function $f$ from a network sample. In particular, \citet{lovasz2012large} showed that for a graph $G \sim \mathcal{G}(n, f)$, that the moments of subgraph counts from $G$ are moments of the kernel function $f$. Specifically, we have the following Lemma:

\begin{lemma}[\citet{lovasz2012large}] \label{moment_interpretation}
	Suppose that $G \sim \mathcal{G}(n, f)$ and let $F$ be a possible subgraph configuration of $G$, namely $|F| \leq |G|$. Define $\mu_F(f)$ to be the expectation of $f$ under the configuration $F$. That is, define:
	\begin{equation}\label{eq:mean} \mu_F(f): = \mathbb{E}\left[\prod_{u,v \in F} f(x_u, x_v) \right] \end{equation}
	Let $\widetilde{S}_{F}(G)$ denote the subgraph density of the configuration $F$ in $G$, as defined in (\ref{subgraph_density}). Then,
	$$\mathbb{E}[\widetilde{S}_F(G)] = \mu_F(f).$$
\end{lemma}

The proof of Lemma \ref{moment_interpretation} is straightforward, and is provided in the first part of the proof of Proposition 1 in \cite{maugis2020testing}. Together, Lemma \ref{lem_reduce_property} and Lemma \ref{moment_interpretation} enable the comparison of the subgraph density matrix $\mathbf{S}$ and the mean subgraph density matrix $\mathbf{M}$ and moreover shed light on the relationship between the principal components output by the PCAN and sPCAN methods. These relationships are made clear in Proposition \ref{ref:equivalence} and Theorem \ref{thm:equiv}.

\begin{proposition}\label{ref:equivalence}
	Suppose that $G_i \sim \mathcal{G}(n_i, f)$ and $F$ is a possible subgraph configuration in $G_i$, namely $|F| \leq |G_i|$. Consider a partition of induced subgraphs $\{G_{i,1}, \ldots, G_{i,K}\}$ from $G_i$, so that $|F| \leq |G_{i,k}|$ for all $k$. Let $\widetilde{S}_F(G_i)$ denote the subgraph density of $F$ in $G_i$. Similarly, define $\widetilde{S}_F(G_{i,k})$ as the subgraph density of $F$ in the induced subgraph $G_{i,k}$ for each $k = 1, \ldots K$ and let $\overline{S}_i = \frac{1}{K} \sum_{k = 1}^K \widetilde{S}_F(G_{i,k})$. Then,
	$$\mathbb{E}[\widetilde{S}_F(G_i)] = \mathbb{E}[\overline{S}_i] = \mu_F(f).$$
\end{proposition}

\begin{proof} From Lemma \ref{moment_interpretation} we have $\mathbb{E}[\widetilde{S}_F(G_i)] = \mu_F(f)$. Turning to $\mathbb{E}[\overline{S}_i]$, from Lemma \ref{lem_reduce_property} we have that $G_{i,k} \sim \mathcal{G}(n_{i,k}, f)$. Applying Lemma \ref{moment_interpretation} gives $\mathbb{E}[\widetilde{S}_F(G_{i,k})] = \mu_F(f)$ for all $k = 1, \ldots, K$. It follows that $\mathbb{E}[\overline{S}_i] = \frac{1}{K} \sum_{k = 1}^K \mathbb{E}[\widetilde{S}_F(G_{i,k})] = \mu_F(f)$. This completes the proof.
\end{proof}

\begin{theorem}\label{thm:equiv}[Relationship between PCAN and sPCAN]
Consider a network sample $\mathbf{G}_N = \{G_1, \ldots, G_N\}$ and collection of subgraph configurations $\mathbf{F}_p = \{F_1, \ldots, F_p\}$. Let $n_{i,k}$ denote the number of vertices in the induced subnetwork $G_{i,k}$ from Step 1 of the sPCAN algorithm and let $m_i = min(\{n_{i,k}: k = 1, \ldots, K\})$. Assume that $\mathbf{F}_p$ satisfies 2~$\max_{j} |F_j| \leq min_{i}\{m_i\}$. Let $\mathbf{S}$ be the subgraph density matrix defined in (\ref{eq:stacked}) and let $\mathbf{M}$ be the mean subgraph density matrix defined in (\ref{eq:M}). If $G_i \sim \mathcal{G}(n_i, f_i)$ for all $G_i \in \mathbf{G}_N$, then the following hold for fixed $p < N$
	\begin{enumerate}
		\item The $(j,i)$th entry of $\mathbf{M}$ and the $(j,i)$th entry of $\mathbf{S}$ are equal in expectation. That is, for all $i = 1, \ldots, N$ and $j = 1, \ldots, p$,
		$$\mathbb{E}[\mathbf{M}_{j,i}] = \mathbb{E}[\mathbf{S}_{j,i}]$$
		\item For all $\ell = 1, \ldots, p$,
		$$ \widehat{\gamma}_\ell \rightarrow \mathbb{E}[\widehat{\lambda}_\ell], \hskip 1pc \text{and} \hskip 1pc \widehat{\mathbf{y}}_\ell \rightarrow \mathbb{E}[\widehat{\mathbf{u}}_\ell]$$
		almost surely as $N, K \rightarrow \infty$.
	\end{enumerate}
\end{theorem}

\begin{proof}
	The first result of the theorem is a direct consequence of Proposition \ref{ref:equivalence}. To prove the second statement, let $\mathcal{S} = \mathbb{E}[\mathbf{S}]$ and let $\Sigma_{\mathcal{S}}$ denote the population variance covariance matrix of $\mathbf{S}$. Applying the strong law of large numbers to the first result, we have $\mathbf{M} \rightarrow \mathcal{S}$ almost surely as $K \rightarrow \infty$. Application of the continuous mapping theorem gives $\Sigma_{\mathbf{M}} \rightarrow  \Sigma_{\mathcal{S}}$ as $K, N \rightarrow \infty$. The second result follows.
\end{proof}

Theorem \ref{thm:equiv} reveals that the population embeddings of the sPCAN method are to the population embeddings of PCAN. We note that requiring $K$ to be large is reasonable in practice because one needs to substitute PCAN for sPCAN only in the case that the sample contains large networks.

Lemma \ref{lem_reduce_property} also enables further analysis of the statistical properties of the sPCAN algorithm. Indeed, application of Lemma \ref{lem_reduce_property} yields a central limit theorem for the mean subgraph density matrix $\mathbf{M}$, a result described in Proposition \ref{lem_CLT1}.

\begin{proposition}\label{lem_CLT1}
Consider a network sample $\mathbf{G}_N = \{G_1, \ldots, G_N\}$ and collection of subgraph configurations $\mathbf{F}_p = \{F_1, \ldots, F_p\}$ for which the sPCAN algorithm is applied. Let $n_{i,k}$ denote the number of vertices in the induced subnetwork $G_{i,k}$ from Step 1 of the sPCAN algorithm. Set $\mathbf{n}_i := \{n_{i,k}: k = 1,\ldots, K\}$ and let $m_i = min(\mathbf{n}_i)$.
	
Suppose that $\mathbf{F}_p$ is chosen so that 2~$\max_{j} |F_j| \leq min_{i}\{m_i\}$. Suppose further that $G_i$ is a kernel-based random graph with kernel $f_i$ for all $G_i \in \mathbf{G}_N$, and let $\mathbf{\mu}_{\mathbf{F}_p}(f_i) = \mathbb{E}[\overline{\mathbf{S}}_i] \in \mathbf{R}^p$. Then there exists a unique $p \times p$ positive semi-definite matrix $\Sigma_{\mathbf{n}_i, f_i, \mathbf{F}_p}$, depending on $\mathbf{n}_i, f_i,$ and $\mathbf{F}_p$, such that for fixed fixed $N \geq 1$ and $p < N$,
	\begin{equation}\label{CLT_1}
		\sqrt{K}(\overline{\mathbf{S}}_i - \mathbf{\mu}_{\mathbf{F}_p}(f_i)) \rightarrow Normal(0, \Sigma_{\mathbf{n}_i, f_i, \mathbf{F}_p}), \hskip 2pc i = 1, \ldots, N
		\end{equation}
\noindent in distribution as $K \rightarrow \infty$. It follows that if $\mathbf{n}_1 = \mathbf{n}_2 = \cdots = \mathbf{n}_{N} =: \mathbf{n}^*$ up to permutations and if $f_1 = f_2 = \cdots = f_N =: f^*$ then every column of $\mathbf{M}$ is asymptotically normal (as $K \rightarrow \infty$) with the same variance covariance matrix $\Sigma_{\mathbf{n}^*, f^*, \mathbf{F}_p}$.
\end{proposition}

\begin{proof}
	From Lemma \ref{lem_reduce_property}, we have that the induced subnetworks from the sPCAN algorithm is a network sample that satisfies $G_{i,k} \sim \mathcal{G}(n_{i,k}, f_i)$ for all $k = 1, \ldots, K$. The asymptotic result in (\ref{CLT_1}) follows from the application of Theorem 1 of \cite{maugis2020testing} to the sample $\{G_{i, 1}, \ldots, G_{i, K}\}$. When the sample sizes of the networks in each bin are equivalent up to permutations, we furthermore know from \cite{maugis2020testing} that $\Sigma_{\mathbf{n}_i,f_i,\mathbf{F}_p}$ no longer depends on $i$. This proves the theorem. \end{proof}

Proposition \ref{lem_CLT1} alone is useful to understand the nature of the principal components that are output by the sPCAN algorithm. Its result, however, also enables the large sample asymptotics of the principal components themselves. As long as $p$ is fixed we can show that the principal components also elicit a central limit theorem by utilizing the tools established in \citet{anderson1963asymptotic}. This is described by Theorem \ref{thm:main}.
 
\begin{theorem}\label{thm:main}
Consider a network sample $\mathbf{G}_N = \{G_1, \ldots, G_N\}$ and a collection of subgraph configurations $\mathbf{F}_p = \{F_1, \ldots, F_p\}$ for which the sPCAN algorithm is applied. Let $\widehat{\gamma}_1, \ldots, \widehat{\gamma}_p$ and $\widehat{\mathbf{w}}_1, \ldots, \widehat{\mathbf{w}}_p$ denote the sample eigenvalues and eigenvectors of $\Sigma_{\mathbf{M}}$, respectively, output by the sPCAN algorithm. Suppose that $\mathbf{n}_i: = \{n_{i,k}: k = 1, \ldots, K\}$ is constant across $i$ up to permutations and equal to $\mathbf{n}^*$ and that $m_i = min(\mathbf{n}_i)$. Let $\gamma_1, \ldots, \gamma_p$ and $\mathbf{w}_1, \ldots, \mathbf{w}_p$ denote the population eigenvalues and eigenvectors of the matrix $\Sigma_{\mathbf{n}^*, f, \mathbf{F}_p}$ from Proposition \ref{lem_CLT1}. 

Suppose that $G_i \sim \mathcal{G}(n_i, f)$ for all $G_i \in \mathbf{G}_N$ and let $\mathbf{x}$ denote the collection of all independently and randomly assigned features to the vertices of all graphs in $\mathbf{G}_N$. Suppose further that 2~$\max_{j} |F_j| \leq min_{i}\{m_i\}$, and that $\gamma_1 > \cdots > \gamma_p$. Fix $0 < p < N$ Conditional on $\mathbf{x}$,

\begin{enumerate}
	\item $\sqrt{N}(\widehat{\gamma}_i - \gamma_i) \rightarrow Normal(0, 2\gamma_i^2)$, and
	\vskip 2pc
	\item $\sqrt{N}(\widehat{\mathbf{w}}_i - \mathbf{w}_i) \rightarrow Normal\left(0,\displaystyle\sum_{i' \neq i} \dfrac{\gamma_i \gamma_{i'}}{(\gamma_i - \gamma_{i'})^2} \mathbf{w}_{i'} \mathbf{w}_{i'}^T\right)$,
\end{enumerate}
\noindent as $N,K \rightarrow \infty$.
\end{theorem}

\begin{proof}
	From Lemma \ref{lem_reduce_property}, we have that $G_{i,k}$ are kernel-based random graphs with kernel $f$. It follows that conditional on the latent variables $\mathbf{x} = \{x_i: i \in V_i\}$ the edges within $E_{i,k}$ are mutually independent and edges from $E_{i,k}$ are mutually independent from edges in $E_{i,k'}$ for all $k' \neq k$. Note that each entry of the subgraph count vector $\mathbf{S}_{i,k}$ is a function of the edges in $E_{i,k}$. It follows from independence of $E_{i,k}$ and $E_{i,k'}$ that $\mathbf{S}_{i,k}$ is conditionally independent of $\mathbf{S}_{i,k'}$ for all $k' \neq k$ given $\mathbf{x}$.
	
	Now, under the conditions of the theorem, we have from Proposition \ref{lem_CLT1} that the columns of $\mathbf{M}$ are asymptotically normal with the same covariance matrix $\Sigma_{\mathbf{n}^*, f, \mathbf{F}_p}$. Applying the usual large-sample theory from \cite{anderson1963asymptotic} to the sequence of mutually conditionally independent statistics $\overline{\mathbf{S}}_1, \ldots, \overline{\mathbf{S}}_{N}$ gives the result. \end{proof}

Theorem \ref{thm:main} provides a framework for which two samples of networks can be formally compared via a hypothesis test that compares the principal components of each sample. 

\section{Applications}\label{sec:application}

We assessed the utility of the PCAN algorithm through two real and diverse applications. In the first application, we analyzed a sample of functional connectivity matrices of 50 patients with schizophrenia and 72 healthy controls collected from resting state fMRI. The goal of this study was to evaluate the topological differences in connectivity between schizophrenia patients and healthy controls. In the second case study, we used the PCAN to explore the dynamics of co-voting behavior between Republican and Democrat senators in the U.S. Congress from 1867 to 2015 (Congress 40 - Congress 113). We investigated what network features of the Senate co-voting networks were predictive of polarization in Congress. We analyzed each data set in the same manner using a strategy that can be generally applied by practitioners for any network sample. Here we outline the analysis strategy used for each application. Then we provide details of the analysis and results in Sections \ref{sec:study1} and \ref{sec:study2}. In Section \ref{sec:study3}, we end by comparing the results of sPCAN with PCAN on these applications. \\

\noindent \underline{\bf Analysis Strategy}\\

We first ran PCAN to identify all principal components from the network sample using the following nine subgraph configurations: isolates, k-star configurations for $k = 1, 2, 3, 4, 5$, triangles, four-cycles (squares), and five-cycles (pentagons). See Figure \ref{fig:examples} for visualizations of these configurations. We then used the estimated principal components for two main tasks - exploratory analysis and predictive analysis, described below:
\begin{itemize}
	\item {\bf Exploratory Analysis}: We first assessed the proportion of variability explained by each principal component (PC) and visualized these results in a scree plot. We next evaluated the contribution of each subgraph configuration to each principal component. The percentage contribution of the configuration was calculated as the relative absolute value of the loading to the total magnitude of the loadings. These analyses enabled the exploration of how much variability each subgraph configuration contributed to the network sample. 
	Since we were interested in discriminating between two collections of observations in each network sample (schizophrenia vs. healthy in functional connectivity and polarized vs. not polarized in co-voting networks), we plotted and analyzed the distributions of each principal component scores for each sub-collection. This part of the analysis provided an understanding of what subgraph configurations contributed the most to the differences in the two sub-collections in each sample. 
	\item {\bf Predictive Analysis}: We next used the top principal components as features in a predictive logistic regression model, where the outcome of interest was a binary variable representing collection type (schizophrenia = 1 vs. healthy = 0, and polarized era = 1 vs. non-polarized era = 0). We evaluated the statistical significance of the principal components using likelihood ratio tests, and furthermore assessed the predictive ability of the principal components using the receiver operating characteristics (ROC) curves and the area under the curve (AUC) metric.
\end{itemize}

\subsection{Analysis of Functional Connectivity Data}\label{sec:study1}
We analyzed a network sample containing functional connectivity matrices from 50 schizophrenia patients (40 paranoid type and 10 residual type), and 72 healthy controls. Preprocessed fMRI data are publicly available from the 1000 connectomes project by the Center for Biomedical Research Excellence (COBRE) (\url{http://fcon\_1000.projects.nitrc.org/indi/retro/cobre.html}). To construct the network sample of functional connectivity matrices, we use a previously validated atlas -- the power atlas \citep{power2011functional} -- that specifies 264 spheres of radius 8mm, which constitute our 264 regions of interest (ROIs). We averaged the fMRI time series from all voxels within each ROI, yielding 264 time series per participant. For each of these time series, we regressed out 6 motion parameters accounting for head movement, 4 parameters corresponding to cerebrospinal fluid, and 4 parameters corresponding to white matter. For each individual in the study, we constructed a functional connectivity network by first calculating the Pearson correlation of each pair of ROIs. We placed an edge between all ROIs that had a correlation in the top 10\% of correlation between pairs of ROIs in the matrix. This gave a network sample of 122 functional connectivity networks of which 72 were healthy controls and the remaining 50 were patients with schizophrenia. More information on the collection, scanning parameters, and preprocessing of the COBRE data is provided in the Appendix.

\subsubsection{Exploratory Analysis}

We applied PCAN to the sample of 122 functional connectivity brain networks in two ways: we considered a whole brain analysis, where each observation in the sample was the entire network of 264 ROIs, as well as a functional subnetwork analysis, where each observation in the sample was the subnetwork containing only the ROIs associated with that subnetwork's function of the brain. We considered ten disjoint functional subnetworks, including the default mode network, salience network, hand network, and fronto-parietal network, among others. We were particularly interested in analyzing functional subnetworks separately because they have been traditionally analyzed in studies of functional connectivity and its relationship to disease (see e.g., \cite{Seeley2007, Menon2010}).

Scree plots of the principal components identified by PCAN from the whole brain sample and each functional subnetwork sample are shown in Figure \ref{fig:fmri_scree}. The plots in Figure \ref{fig:fmri_scree} convey a consistent story across subnetworks and the whole brain. First, in all cases, the first pc explains at least 80$\%$ of the variability in the sample. The first and second PCs combined explains at least 92$\%$ of the variability in the network sample, and in all functional subnetworks the first three PCs explain 99$\%$ or more of the variability in the sample. These results suggest that the patterns inherent to the functional connectivity of these participants, regardless of the subnetwork considered, can be largely explained with just three dimensions. 

\begin{figure}[htbp!]
    \centering
    \begin{subfigure}[t]{0.31\textwidth} 
        \centering
        \includegraphics[width = 0.8\linewidth]{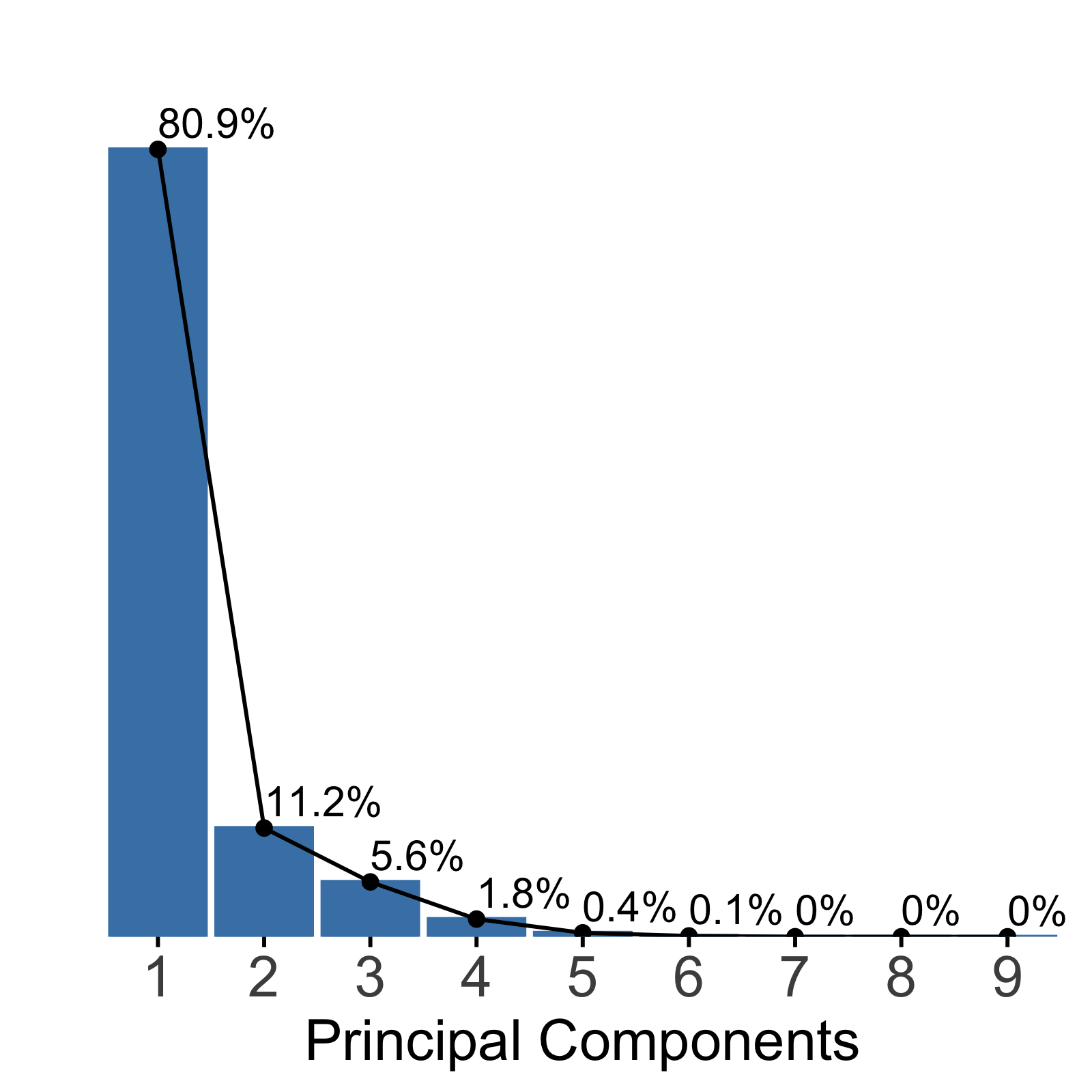} \caption{Whole-Brain}\label{fig:whole_scree}
    \end{subfigure} \hfill
    \begin{subfigure}[t]{0.31\textwidth}
        \centering
      \includegraphics[width = 0.8\linewidth]{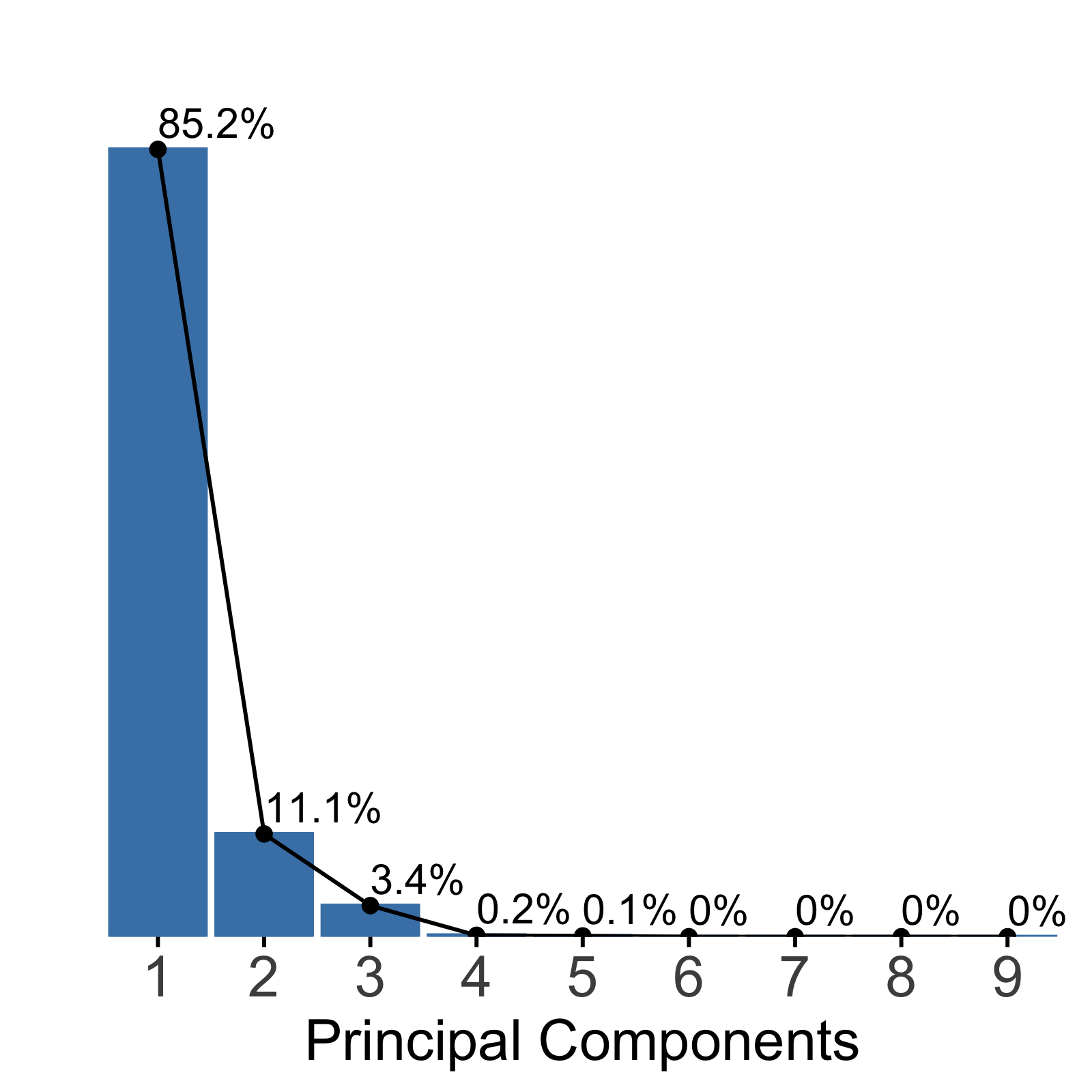} \caption{Default Mode Network}
    \end{subfigure} \hfill
    \begin{subfigure}[t]{0.31\textwidth}
        \centering
       \includegraphics[width = 0.8\linewidth]{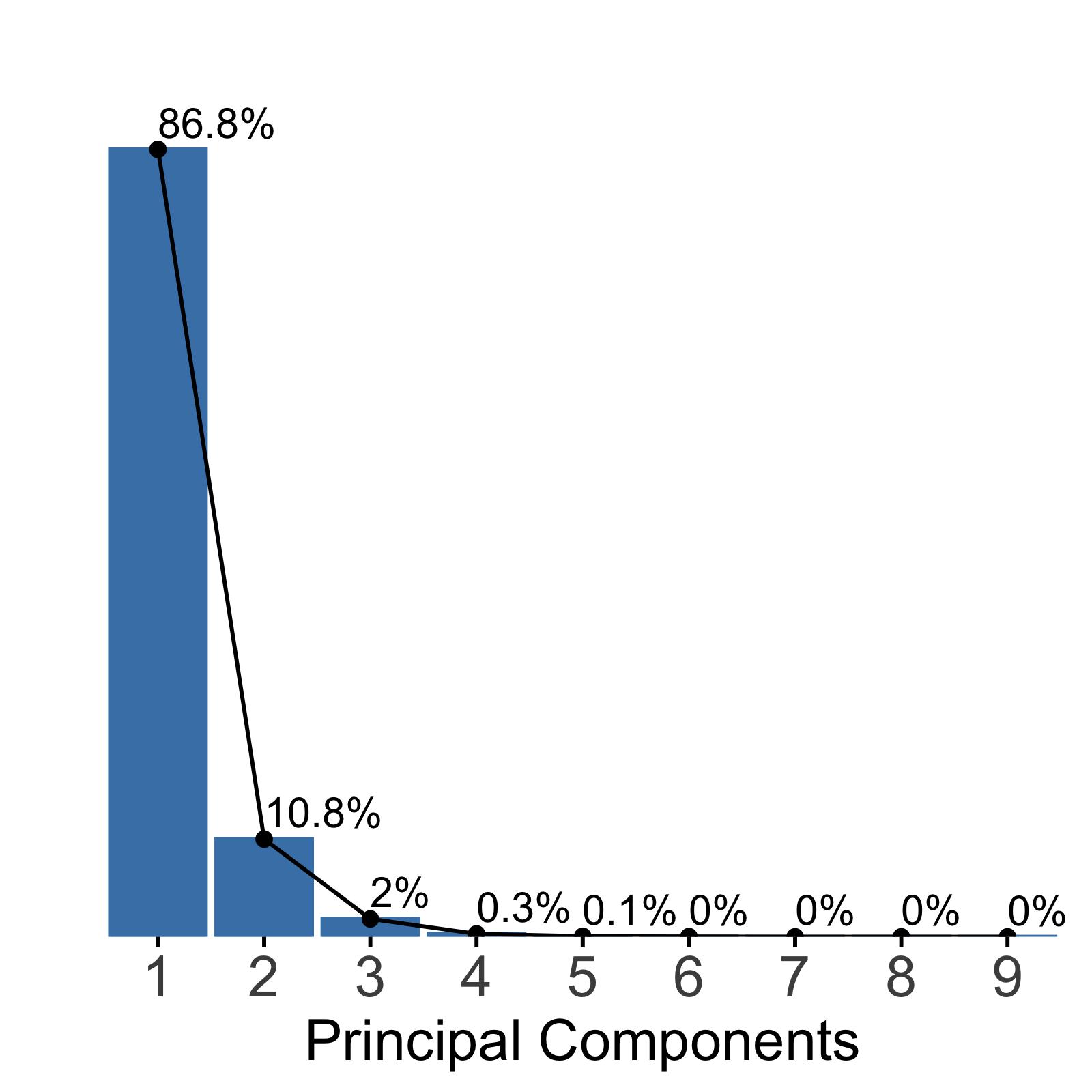} \caption{Salience Network}
    \end{subfigure} \hfill
    \begin{subfigure}[t]{0.31\textwidth}
        \centering
       \includegraphics[width = 0.8\linewidth]{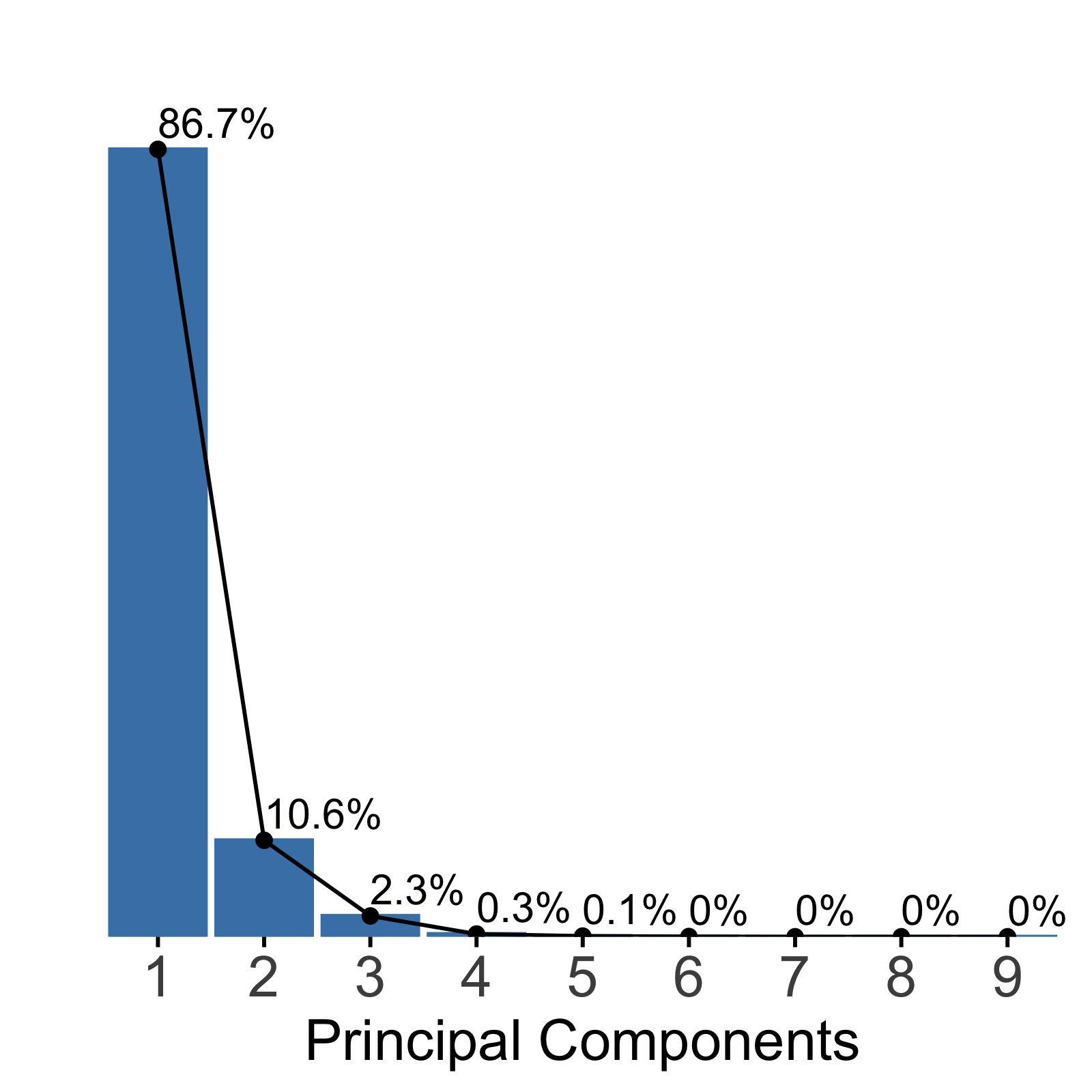} \caption{Cing. Oper. Task Control}
    \end{subfigure} \hfill    
    \begin{subfigure}[t]{0.31\textwidth}
        \centering
       \includegraphics[width = 0.8\linewidth]{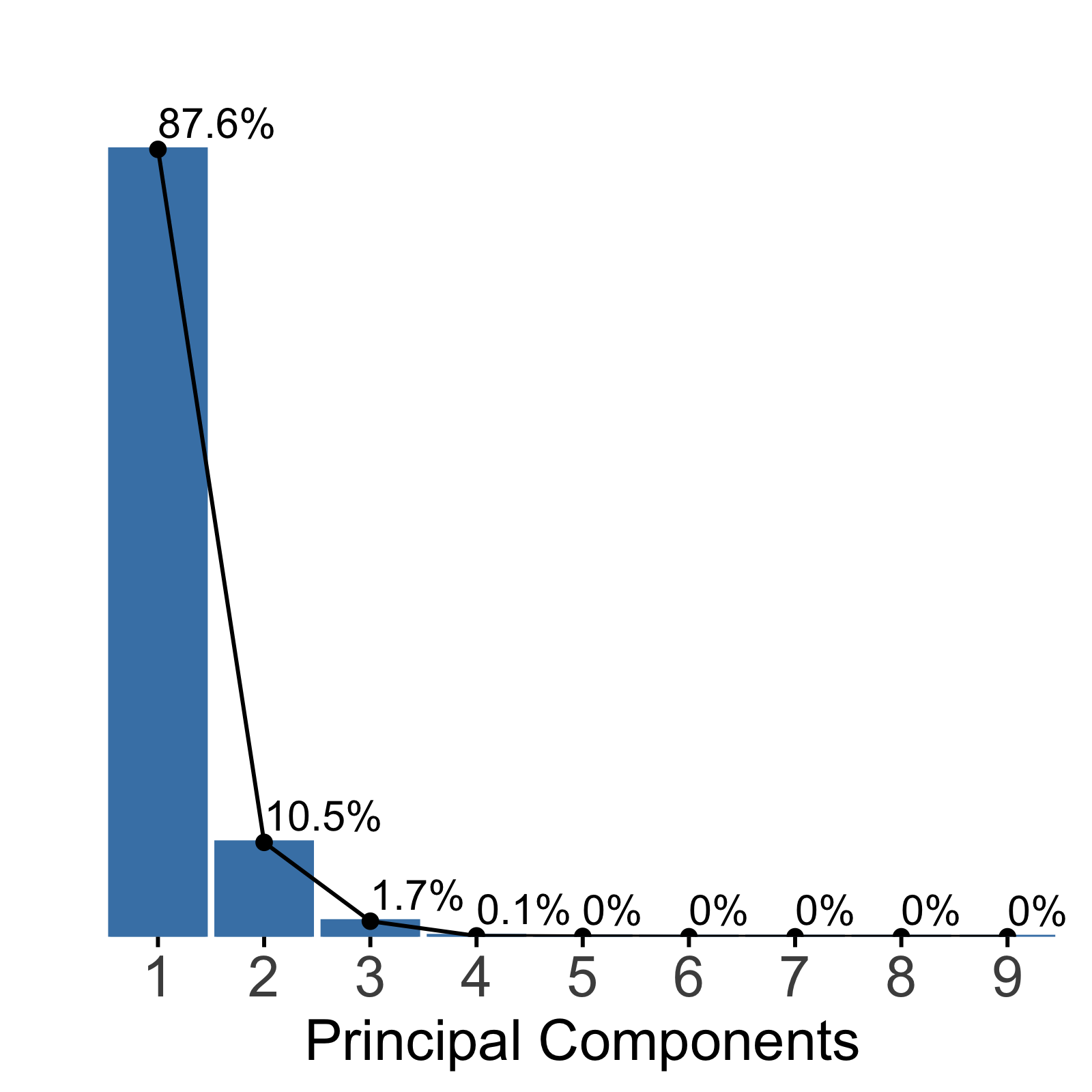} \caption{Auditory}  
    \end{subfigure} \hfill
    \begin{subfigure}[t]{0.31\textwidth}
        \centering
       \includegraphics[width = 0.8\linewidth]{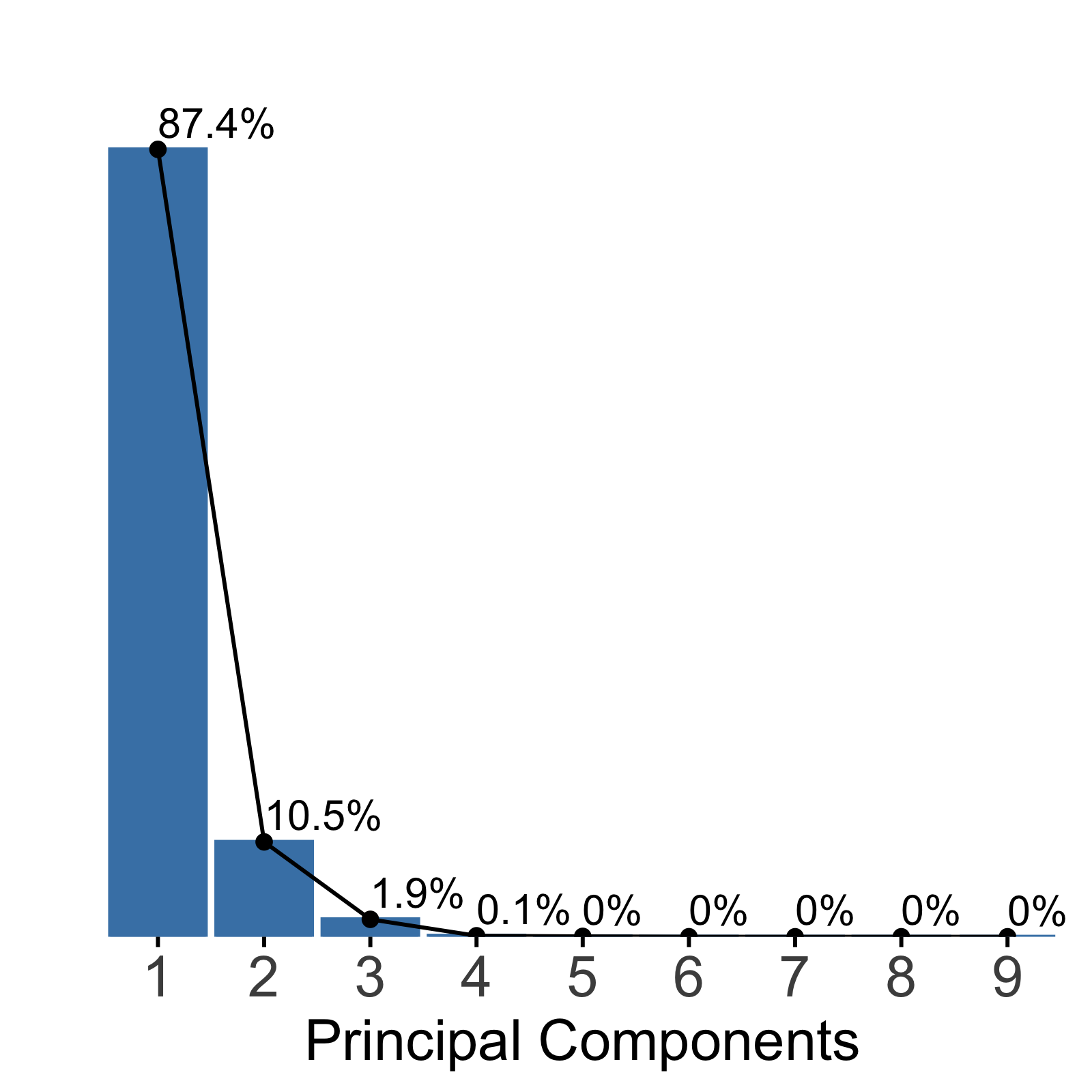} \caption{Hand}
    \end{subfigure} \hfill    
    \begin{subfigure}[t]{0.31\textwidth}
        \centering
       \includegraphics[width = 0.8\linewidth]{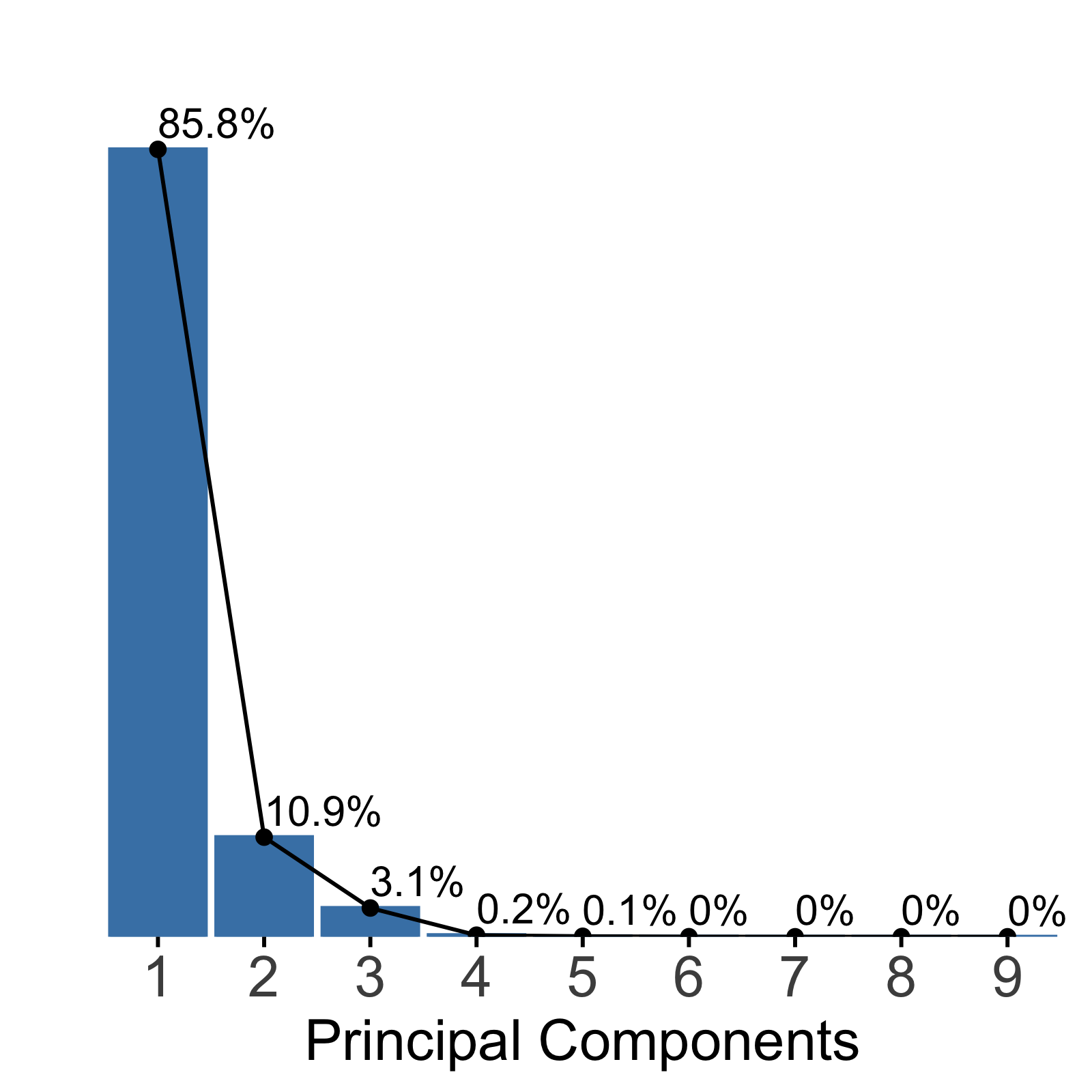} \caption{Fronto Parietal Task Control}
    \end{subfigure} \hfill    
    \begin{subfigure}[t]{0.31\textwidth}
        \centering
       \includegraphics[width = 0.8\linewidth]{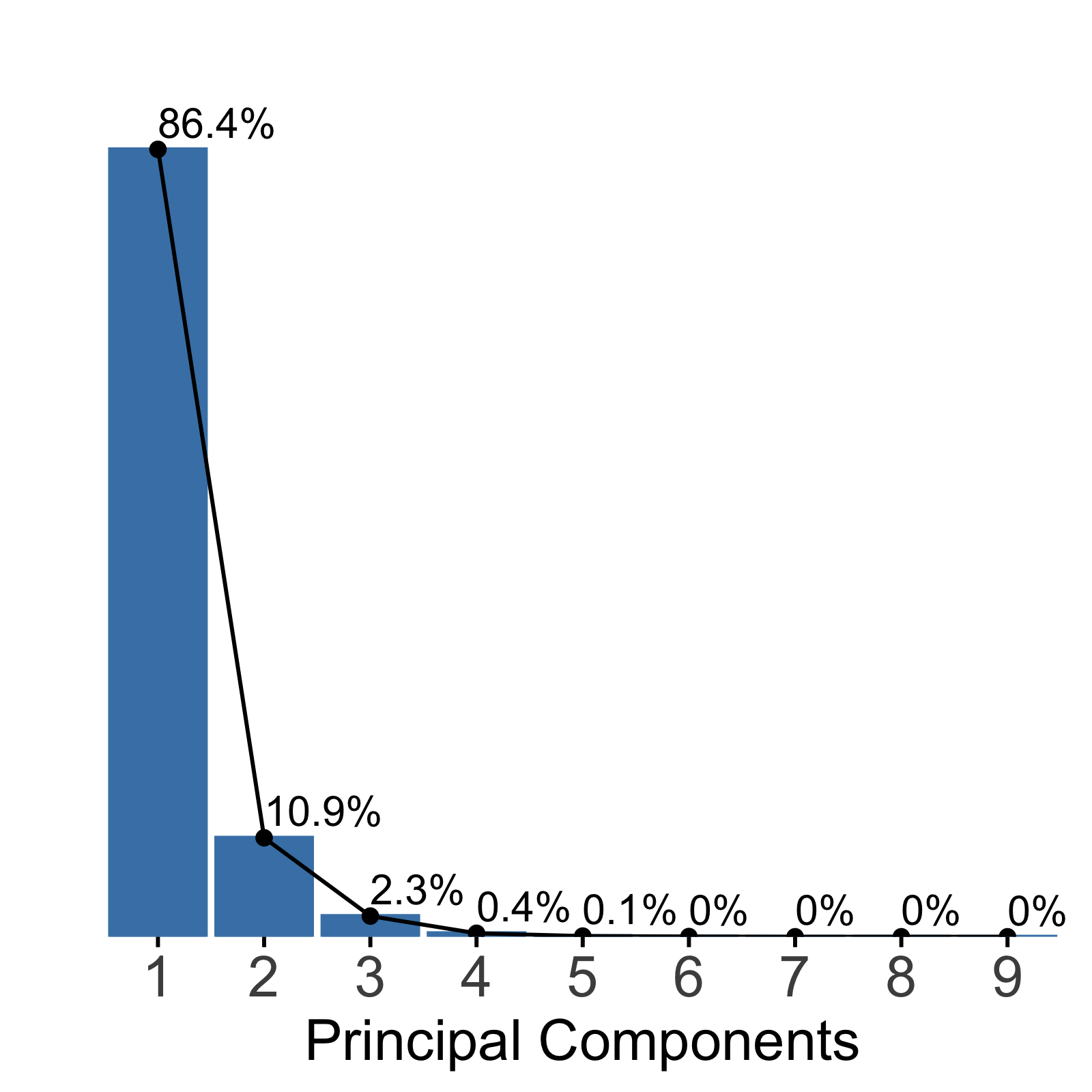} \caption{Subcortical}        		
    \end{subfigure} \hfill
    \begin{subfigure}[t]{0.31\textwidth}
        \centering
       \includegraphics[width = 0.8\linewidth]{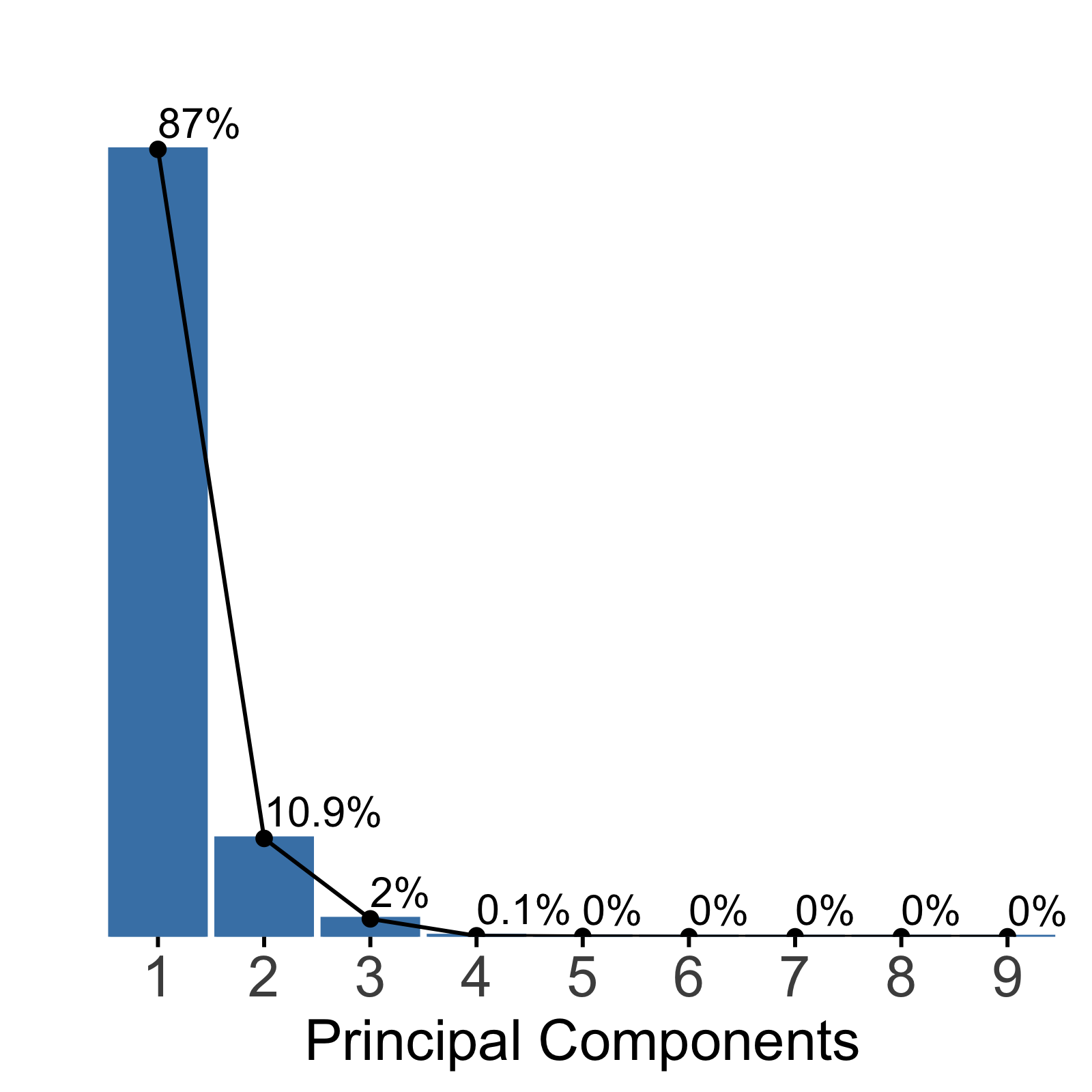} \caption{Visual}
    \end{subfigure} \hfill
    \begin{subfigure}[t]{0.31\textwidth}
        \centering
       \includegraphics[width = 0.8\linewidth]{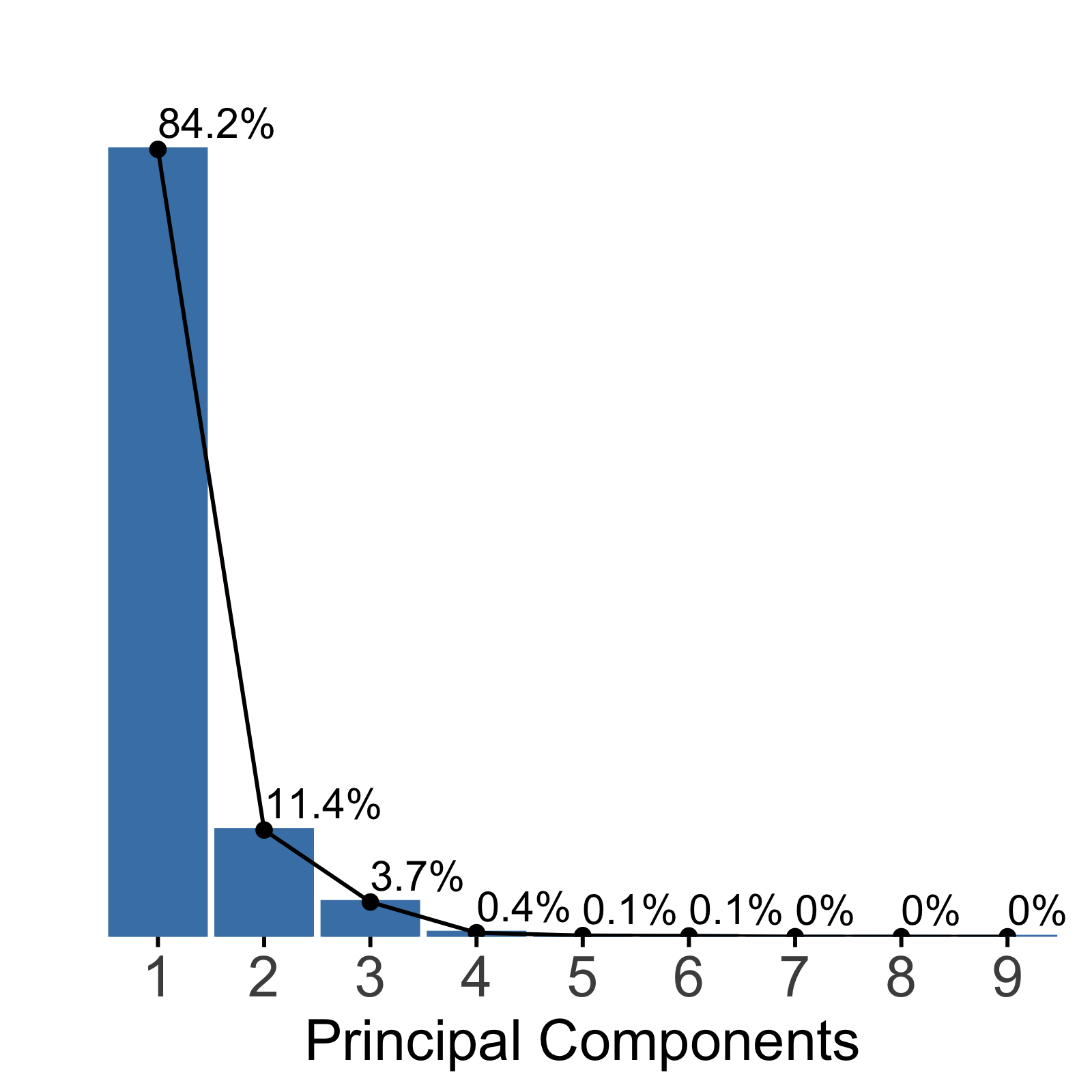}	\caption{Ventral Attention}
    \end{subfigure} \hfill   
    \begin{subfigure}[t]{0.31\textwidth}
        \centering
       \includegraphics[width = 0.8\linewidth]{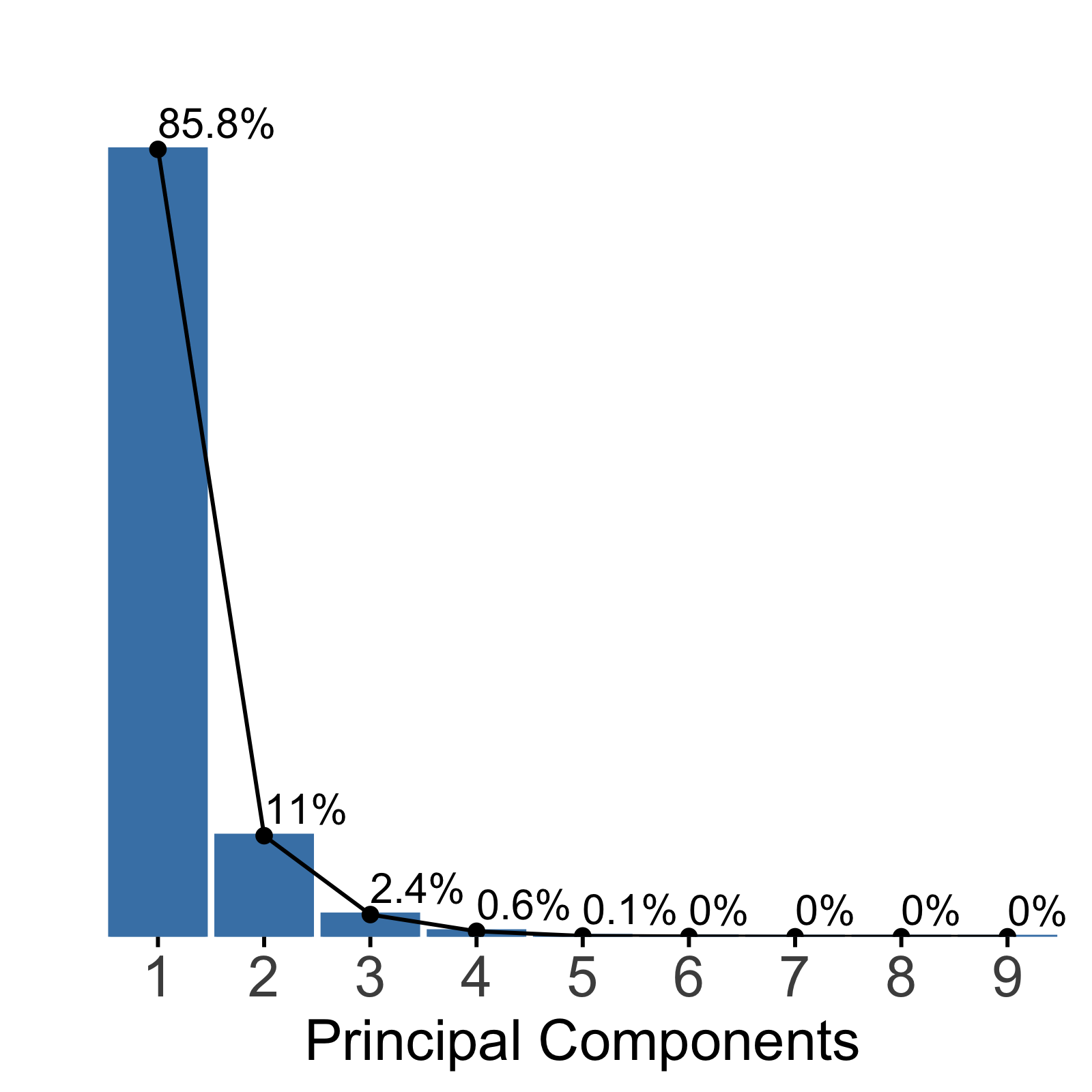} \caption{Dorsal Attention}
    \end{subfigure} \hfill           \vfill
    \caption{Scree plots for the whole-brain analysis as well as and functional subnetwork analysis in the functional connectivity study. \label{fig:fmri_scree}}
\end{figure}

We next explored the contributions of each subnetwork configuration on the identified principal components. Figure \ref{fig:fmri_contributions} shows the contributions of each configuration on the top four PCs for the whole brain analysis as well as the top three PCs for the hand and salience subnetworks. We focus our discussion on the hand and salience networks because previous studies have identified these two subnetworks as containing important connectivity signatures that are predictive of schizophrenia \citep{Seeley2007, Menon2010, wilson2021analysis}. From Figure \ref{fig:fmri_contributions}, we first find that high order subgraph configurations (containing 3 or more nodes) contributed fairly equally to the first principal component, while low order configurations (containing 1 or 2 nodes) contributed the least. This result was consistent across all functional subnetworks. This suggests that configurations of higher order are important features in describing the most variability of the network sample. There were notable differences between the subnetwork analysis and whole brain analysis for the second and third PCs. For example, isolates contributed significantly more to the second PC ($> 95\%$ contribution) for both the hand and salience networks, while the edges configuration was the primary contributor to the second PC in the whole brain sample. The contributions to the third PC were the most strikingly different between the subnetwork analysis and whole brain analysis. Consistent to both the salience and hand networks, edges contributed the most followed by 5-stars, 2-stars, and triangles, each of which contributed more than average. On the other hand, in the whole brain analysis, only the isoloate configuration contributed more than average to the third pc.

The notable differences in the second and third PCs identified in Figure \ref{fig:fmri_contributions} may provide important information that can discriminate schizophrenia patients from healthy controls. To further explore this, we next analyzed the distributions of the principal component scores for schizophrenia patients and healthy controls. We plotted these distributions for the top four PCs for the whole brain analysis and the top three PCs for the hand and salience samples. As suspected, we found that there was some notable differentiation among schizophrenia and healthy controls in the second and third PCs for the hand and salience networks. These differences were not observed in the score distributions of the whole brain analysis. 

Taken together, our exploratory analysis identified three important insights about the functional connectivity network sample considered here. First across all subnetworks and the whole brain, three PCs described the predominant amount of variability ($> 95\%$) in the sample; moreover, the first PC explained greater than $80\%$ of the total variability. Second, the first PC captured consistent characteristics across all subnetworks and the whole brain, and it represented high-order subgraph configurations containing 3 or more nodes. Finally, our analysis revealed that the second and third PCs for the hand and salience subnetworks identified consistent patterns in the sample which could be used to discriminate between schizophrenia patients and healthy controls. We formally test this last point in our predictive analysis next.

\begin{figure}[htbp!] 
    \centering
    \begin{subfigure}[t]{\textwidth} 
        \centering
        \includegraphics[width = \textwidth]{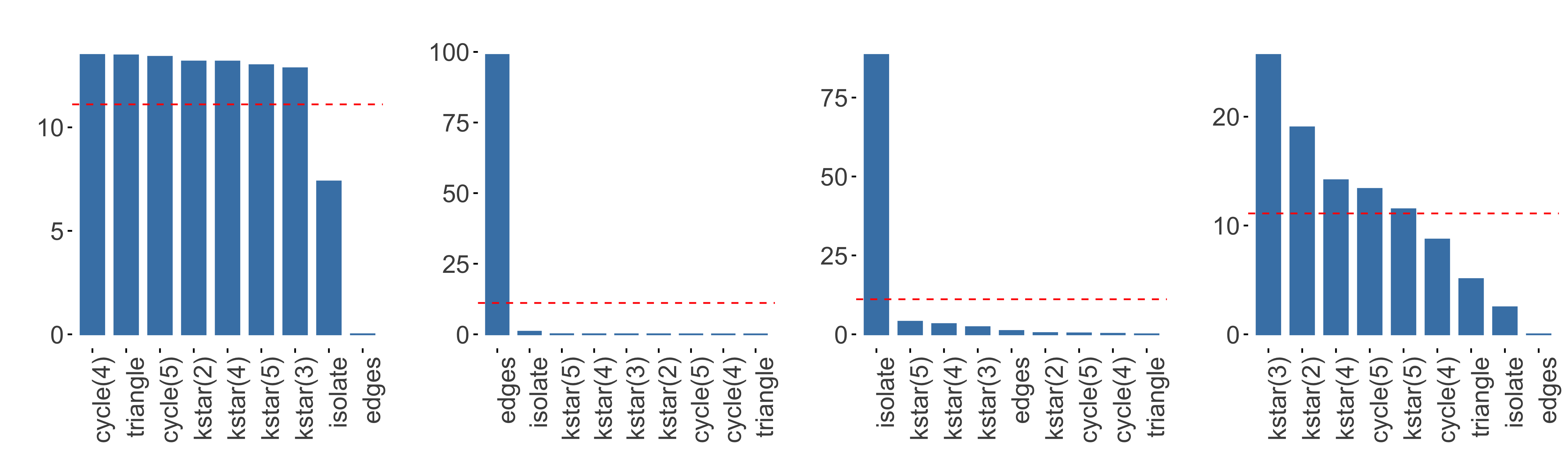} \caption{Whole Brain}
    \end{subfigure} \hfill
    \begin{subfigure}[t]{\textwidth}
        \centering
      \includegraphics[width = \textwidth]{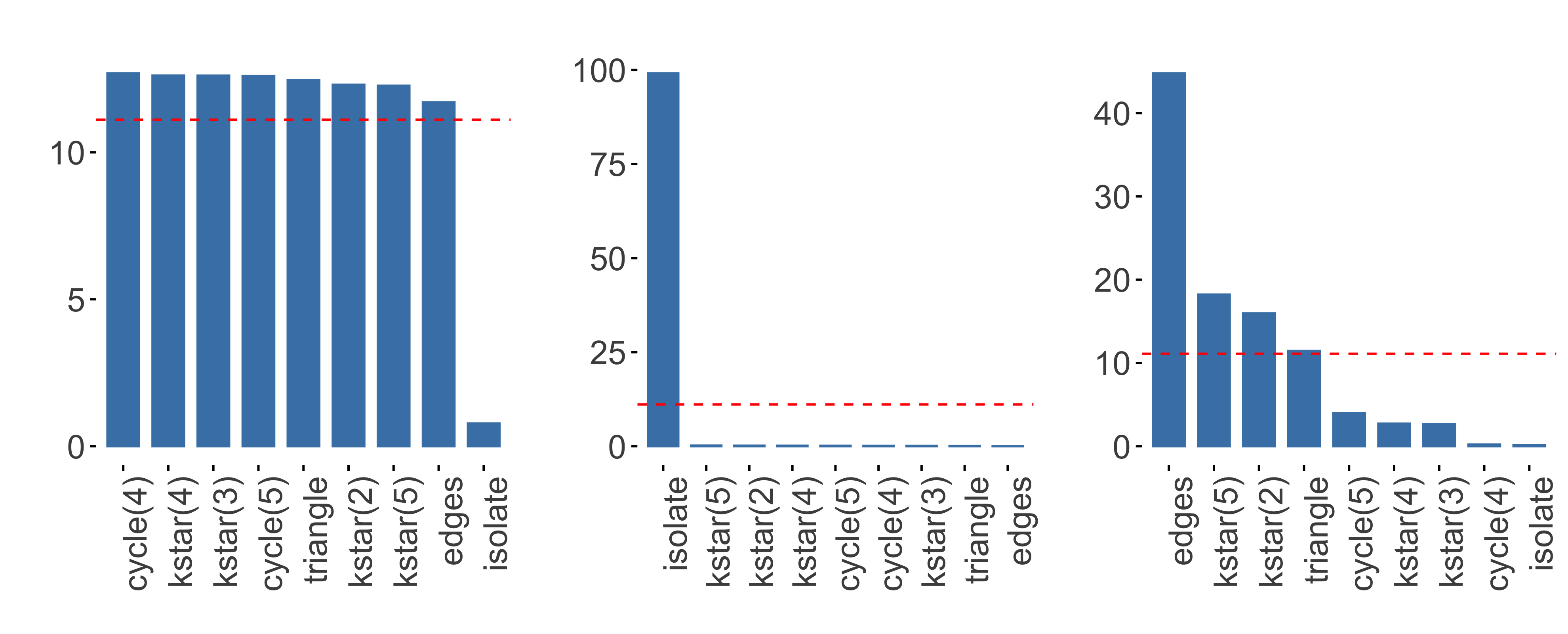} \caption{Hand Network}
    \end{subfigure} \hfill
    \begin{subfigure}[t]{\textwidth}
	        \centering
	\includegraphics[width = \textwidth]{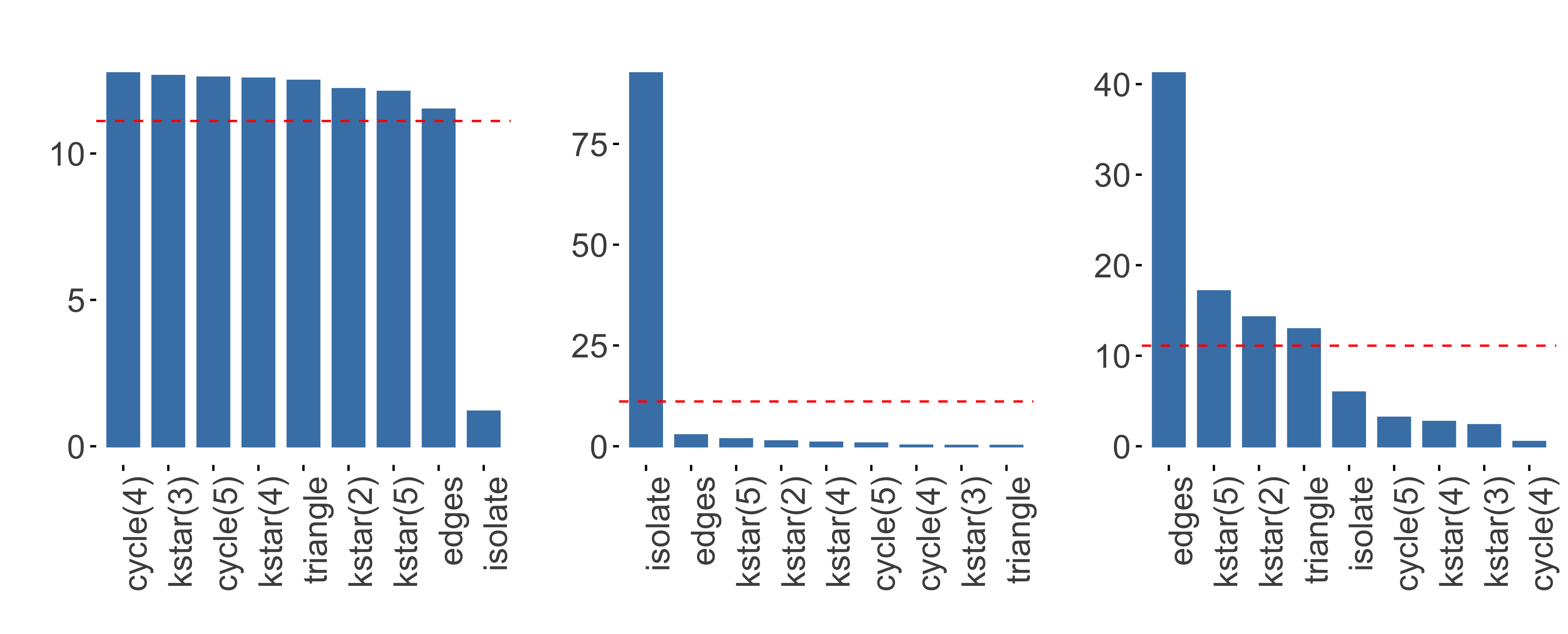} \caption{Salience Network}
	    \end{subfigure}
	\caption{Percentage contributions of the network configurations for the principal components contributing at least 1$\%$ variability for the whole brain sample, the hand network sample, and the salience network sample. The dotted red line shows the average contribution ($= 1/9 * 100\%$) across all 9 configurations. \label{fig:fmri_contributions}}
\end{figure}

\begin{figure}[htbp!] 
    \centering
    \begin{subfigure}[t]{\textwidth} 
        \centering
        \includegraphics[width = \textwidth]{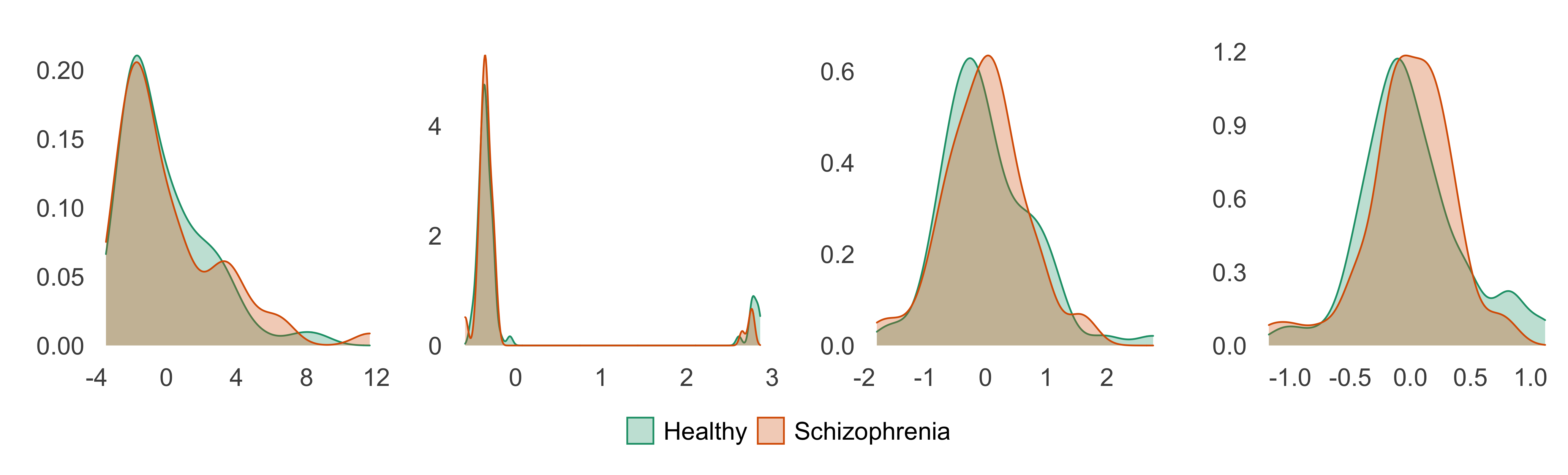} \caption{Whole Brain}
    \end{subfigure} \hfill
    \begin{subfigure}[t]{\textwidth} 
        \centering
      \includegraphics[width = \textwidth]{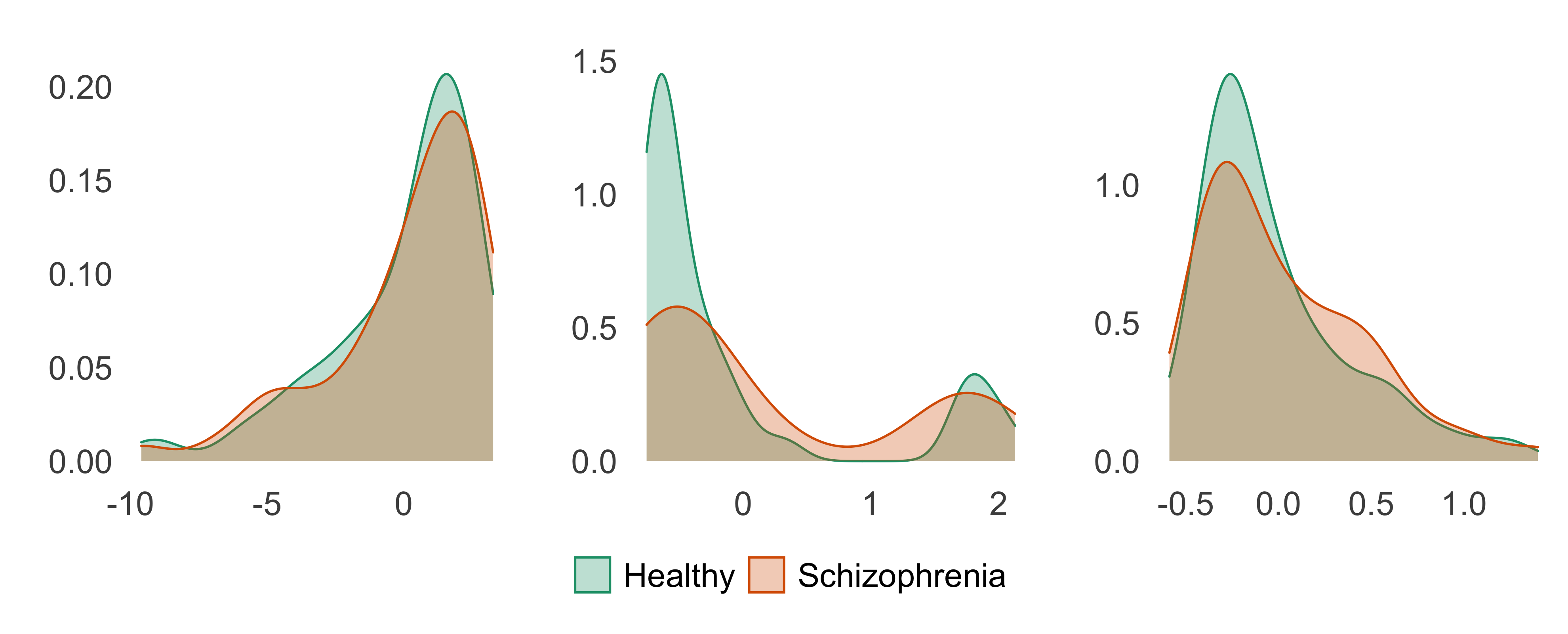} \caption{Hand Network}
    \end{subfigure} \hfill
    \begin{subfigure}[t]{\textwidth} 
	        \centering
	\includegraphics[width = \textwidth]{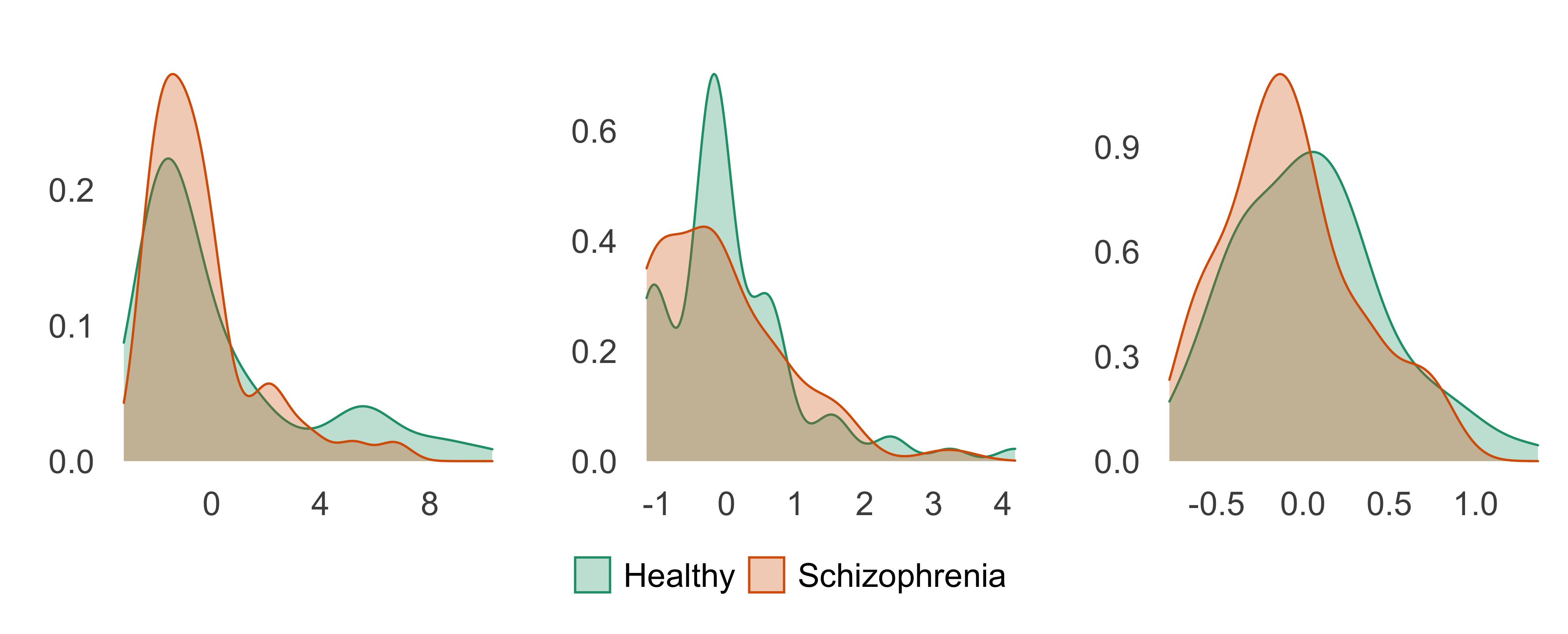} \caption{Salience Network}
	    \end{subfigure}
	\caption{The distributions of the principal component scores for the healthy controls and schizophrenia patients for the principal components contributing at least 1$\%$ variability. Shown are the distributions for the whole brain sample, the hand network sample, and the salience network sample. \label{fig:fmri_scores}}
\end{figure}

\subsubsection{Predictive Analysis}

We next fit logistic regression models to classify schizophrenia patients from healthy controls. To assess the predictive ability of the principal components, we compared five regression models. In each of the models, the outcome of interest was the binary variable set to 1 for schizophrenia patients and 0 for healthy controls. In {\bf Model 1} we used the top 4 principal components identified for the whole brain analysis as features. We chose 4 PCs to ensure that each PC explained at least 1$\%$ variability in the sample. In {\bf Model 2} we used the top 3 principal components identified for each of the ten functional subnetwork samples. We chose three here again to ensure that each PC explained at least 1$\%$ variability. For {\bf Model 3}, we fit an L2-penalized / Ridge logistic regression model with the same features from {\bf Model 2}. {\bf Model 4} contained the raw standardized subgraph configuration densities from the whole brain, and {\bf Model 5} was a Ridge model containing the raw standardized subgraph configuration densities from each of the ten subnetworks of the brain. We reported the odds ratio, confidence intervals, and fit metrics (AUC, AIC, BIC) for {\bf Model 1} and {\bf Model 2} in Tables \ref{tab:fmri_whole} and \ref{tab:fmri_subnetworks}, respectively. Figure \ref{fig:fmri_roc} compared receiver operating characteristic (ROC) curves and area under the curve (AUC) from the five models. 

\begin{table}[htbp!]
\centering
\begin{tabular}{l|rrrr}
\toprule
\multicolumn{1}{c|}{PC} & \multicolumn{1}{c}{Variability} & \multicolumn{1}{c}{OR} & \multicolumn{1}{c}{95\% CI} & \multicolumn{1}{c}{p-value} \\
\hline
PC1                    & 80.9\%                          & 1.03                   & (0.9, 1.19)                 & 0.618                       \\
PC2                    & 11.2\%                          & 0.82                   & (0.53, 1.18)                & 0.308                       \\
PC3                    & 5.6\%                           & 0.94                   & (0.56, 1.57)                & 0.811                       \\
PC4                    & 1.8\%                           & 0.88                   & (0.35, 2.15)                & 0.775                      \\
\bottomrule
\end{tabular} 
\caption{A summary of the logistic regression model fit on the principal components identified for the whole brain to predict schizophrenia (AUC = 0.505; AIC = 173.6; BIC = 187.7). \label{tab:fmri_whole}}
\end{table}

\begin{table}[htbp!]
\centering
\begin{tabular}{l|l|rrcr}
\toprule
\multicolumn{1}{c|}{Subnetwork} & \multicolumn{1}{c|}{PC} & \multicolumn{1}{c}{Variability} & \multicolumn{1}{c}{OR} & \multicolumn{1}{c}{95 \% CI} & \multicolumn{1}{c}{p-value} \\
\hline
Auditory                       & PC1                    & 87.6\%                          & 0.98                   & (0.8, 1.21)                 & 0.852                       \\
                               & PC2                    & 10.5\%                          & 1.03                   & (0.65, 1.67)                & 0.904                       \\
                               & PC3                    & 1.7\%                           & 1.13                   & (0.34, 3.85)                & 0.838                       \\ \hline
Cing. Oper. Task Control       & PC1                    & 86.7\%                          & 0.86                   & (0.69, 1.06)                & 0.174                       \\
                               & PC2                    & 10.6\%                          & 1.23                   & (0.75, 2.07)                & 0.425                       \\
                               & PC3                    & 2.3\%                           & 0.83                   & (0.28, 2.43)                & 0.725                       \\ \hline
Default Mode Network           & PC1                    & 85.2\%                          & 0.97                   & (0.79, 1.2)                 & 0.754                       \\
                               & PC2                    & 11.1\%                          & 0.92                   & (0.51, 1.66)                & 0.786                       \\
                               & PC3                    & 3.4\%                           & 0.7                    & (0.26, 1.65)                & 0.434                       \\ \hline
Dorsal Attention               & PC1                    & 85.8\%                          & 1.13                   & (0.95, 1.37)                & 0.182                       \\
                               & PC2                    & 11\%                            & 1.18                   & (0.75, 1.9)                 & 0.485                       \\
                               & PC3                    & 2.4\%                           & 0.83                   & (0.29, 2.28)                & 0.725                       \\ \hline
Fronto Parietal Task Control   & PC1                    & 85.8\%                          & 0.74                   & (0.38, 1)                   & 0.255                       \\
                               & PC2                    & 10.9\%                          & 1.03                   & (0.47, 2.01)                & 0.927                       \\
                               & PC3                    & 3.1\%                           & 2.36                   & (0.68, 17.59)               & 0.331                       \\ \hline
Hand                           & PC1                    & 87.4\%                          & 1.08                   & (0.88, 1.33)                & 0.450                        \\
                               & PC2                    & 10.5\%                          & 1.57                   & (0.98, 2.6)                 & {\bf 0.068}                       \\
                               & PC3                    & 1.9\%                           & 2.28                   & (0.68, 7.94)                & 0.183                       \\ \hline
Salience                       & PC1                    & 86.8\%                          & 0.83                   & (0.63, 1.03)                & 0.120                        \\
                               & PC2                    & 10.8\%                          & 0.67                   & (0.37, 1.17)                & 0.169                       \\
                               & PC3                    & 2\%                             & 0.31                   & (0.08, 1.03)                & {\bf 0.069}                       \\ \hline
Subcortical                    & PC1                    & 86.4\%                          & 1.01                   & (0.85, 1.2)                 & 0.897                       \\
                               & PC2                    & 10.9\%                          & 1.23                   & (0.77, 2.03)                & 0.399                       \\
                               & PC3                    & 2.3\%                           & 1.29                   & (0.44, 3.99)                & 0.644                       \\ \hline
Ventral Attention              & PC1                    & 84.2\%                          & 0.9                    & (0.72, 1.07)                & 0.261                       \\
                               & PC2                    & 11.4\%                          & 0.91                   & (0.57, 1.43)                & 0.685                       \\
                               & PC3                    & 3.7\%                           & 0.84                   & (0.26, 2.05)                & 0.719                       \\ \hline
Visual                         & PC1                    & 87\%                            & 1.02                   & (0.85, 1.21)                & 0.826                       \\
                               & PC2                    & 10.9\%                          & 0.83                   & (0.49, 1.37)                & 0.464                       \\
                               & PC3                    & 2\%                             & 1.73                   & (0.59, 5.57)                & 0.336                      \\
\bottomrule
\end{tabular}
\caption{A summary of the logistic regression model fit on the top three principal components identified for functional subnetworks to predict schizophrenia (AUC = 0.775; AIC = 198.3; BIC = 285.2). Bolded p-values are statistically significant at significance level 0.10. \label{tab:fmri_subnetworks}}
\end{table}

From Table \ref{tab:fmri_subnetworks}, we found that the third PC from the salience subnetwork and the second PC from the hand subnetwork had statistically significant effects on the prediction of schizophrenia ($p<0.1$). This coincides with our earlier exploratory analysis and furthermore provides support to the Triple Network Theory from \cite{Menon2010, menon2011large} which hypothesizes that pathological salience evident in patients with schizophrenia may be associated with pathology and dysregulation in the default mode network (dmn) and hand subnetworks. Transitioning to the whole brain analysis, we found that none of the PCs were statistically significant.

When comparing the five models in terms of predictive performance, we found that the model containing PCs from all functional subnetworks ({\bf Model 2}) gave the highest AUC of 0.775. The second best model was the L2 penalized version of this model ({\bf Model 3}), which obtained an AUC of 0.741. Notably, the AUC of these two models is much higher than the PCs derived from the whole brain (Model 1, AUC = 0.505). This result suggests that the variability that best differentiates patients with schizophrenia from healthy controls is localized to functional subnetworks rather than to the whole brain. This finding was reinforced with the results of significant effects in the hand and salience subnetworks previously discussed. In comparing our classification results with other studies on this data set, we found that both {\bf Model 2} and {\bf Model 3} are competitive with state-of-the-art predictive models \citep{relion2019network, wilson2020hierarchical, wilson2021analysis}. This is particularly notable because the PCs themselves were identified in a fully unsupervised fashion, unlike the top performing classification results obtained in \cite{relion2019network} (accuracy $> 92\%$ on the COBRE data set).

We furthermore found that models that incorporate raw and highly correlated subgraph configurations tended to have worse predictive performance than {\bf Model 2} and {\bf Model 3} ({\bf Model 4}: AUC = 0.635, and {\bf Model 5}: AUC = 0.676). Thus, using the principal components of this complex network sample provides a significant gain in accuracy over describing the sample with the raw subgraph densities themselves.

\begin{figure}[htbp!]
	\centering
	\includegraphics[width = 0.55\textwidth]{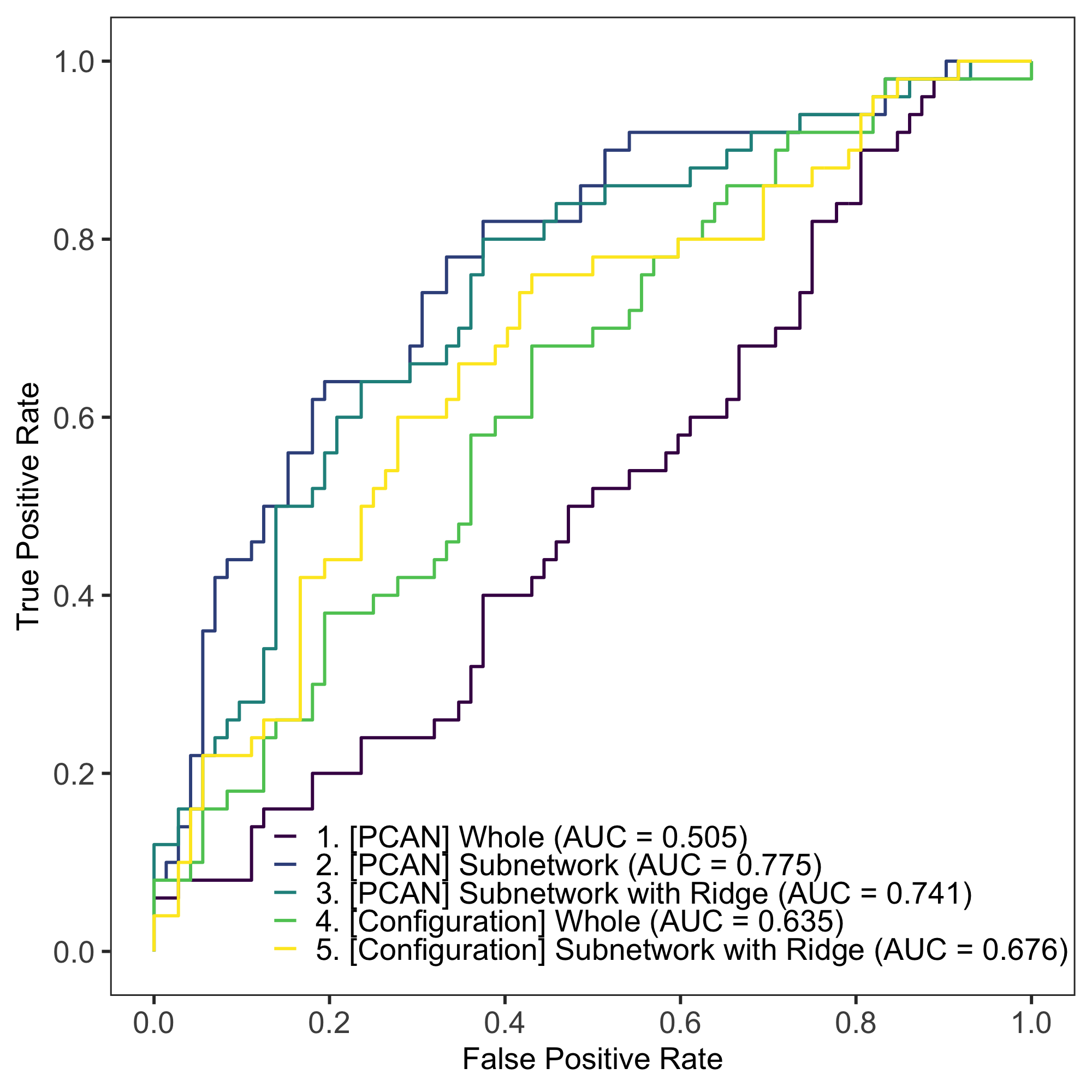}
	\caption{Classification results in the functional connectivity study. Shown are the ROC curves for the five fitted logistic regression models for classifying patients with schizophrenia from healthy controls. \label{fig:fmri_roc}} 
\end{figure}

\subsection{Analysis of Senate Co-Voting Data}\label{sec:study2}

In the past two decades, the U.S. public has become noticeably more politically polarized. Recent analyses of this co-voting data have found that the network summaries like modularity and co-voting propensities within and between parties reflect the onset of this polarization \citep{moody2013portrait, lee2020varying, sparks2019monitoring, wilson2019modeling}. In this study, we sought to characterize the differences of the Senate co-voting behavior during the ``polarization'' era from their behavior in the  ``non-polarized'' era. Such an analysis could provide a deeper understanding of the relationships among senators in these two eras. 

The Senate co-voting network sample describes the co-voting patterns among U.S. Democrat and Republican senators from 1867 (Congress 40) to 2015 (Congress 113). The network was constructed from publicly available roll call voting data from  \url{http://voteview.com}, which contains the voting decision of each senator (yay, nay, or abstain) for every bill brought to Congress.  We modeled the co-voting tendencies of the senators using a dynamic network where nodes represent senators and an edge is formed between two nodes if the two senators vote concurringly (both yay or both nay) on at least 75\% of the bills to which they were both present. 

To investigate the differences in the co-voting behavior before and during the polarization eras, we labeled networks from Congress 40 - 99 as taking part of the ``non-polarized'' era, and labeled networks from Congress 100 - 113 as taking part of the ``polarized'' era. The split of these two samples at Congress 100 represents the theoretical and empirical evidence that polarization in the Senate emerged around the Clinton administration (Congress 103) \citep{moody2013portrait, wilson2019modeling}. 

\subsubsection{Exploratory Analysis}

We ran the PCAN method on three versions of the network sample containing all 73 networks. The first sample, which we call the ``two-party'' network sample, contained observations with all Republican and Democrat senators in each congress. We also analyzed the network samples where each observation contained the induced subgraph corresponding to senators of the same political party. This resulted in a ``Republican'' network sample and a ``Democrat'' network sample. Independent senators were ignored in all network samples. 

Scree plots for the principal components identified by PCAN from each sample are provided in Figure \ref{fig:political_scree}, and the contribution of each network configuration to each of the top three PCs were plotted in Figure \ref{fig:political_contrib}. As seen in Figure \ref{fig:political_scree}, we found that the first principal component explained at least 88$\%$ of the variability in all three network samples. Furthermore, the first two PCs explained 98$\%$ or more of the variability in each network samples. The consistency of these results suggest that the variability in the co-voting network sample can be nearly fully explained by just two dimensions. The contribution plots in Figure \ref{fig:political_contrib} also reveal a consistent pattern across all three network samples. In particular, all configurations except the isolate configuration have roughly equal contributions to the first PC, whereas the isolate configuration contributes less than 4$\%$ in each sample. Furthermore, the second PC is predominately explained by the isolate configuration ($> 80\%$) in each sample, and all other configurations do not meaningfully contribute. Finally, the edges, 2-star, 5-star, and triangle configurations each contribute more than average in the third principal component. Together, these results suggest that the two-party, Republican, and Democrat co-voting behavior have similar topological structure across the time period analyzed.

\begin{figure}[htbp!] 
    \centering
    \begin{subfigure}[t]{0.31\textwidth} 
        \centering
        \includegraphics[width = 0.8\linewidth]{Whole1.png} \caption{Two-party}
    \end{subfigure} \hfill
    \begin{subfigure}[t]{0.31\textwidth}
        \centering
      \includegraphics[width = 0.8\linewidth]{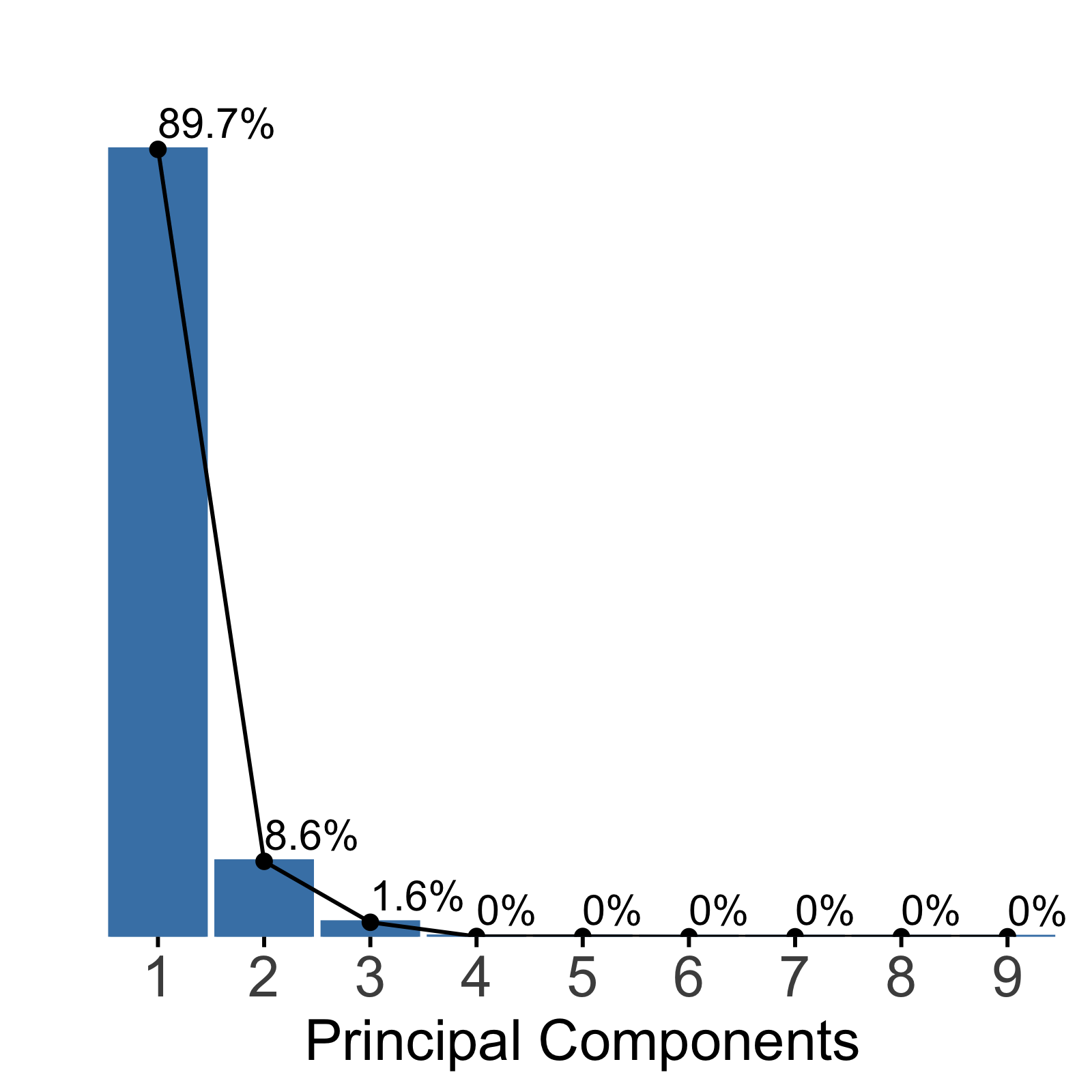} \caption{Democrat}
    \end{subfigure} \hfill
    \begin{subfigure}[t]{0.31\textwidth}
	        \centering
	      \includegraphics[width = 0.8\linewidth]{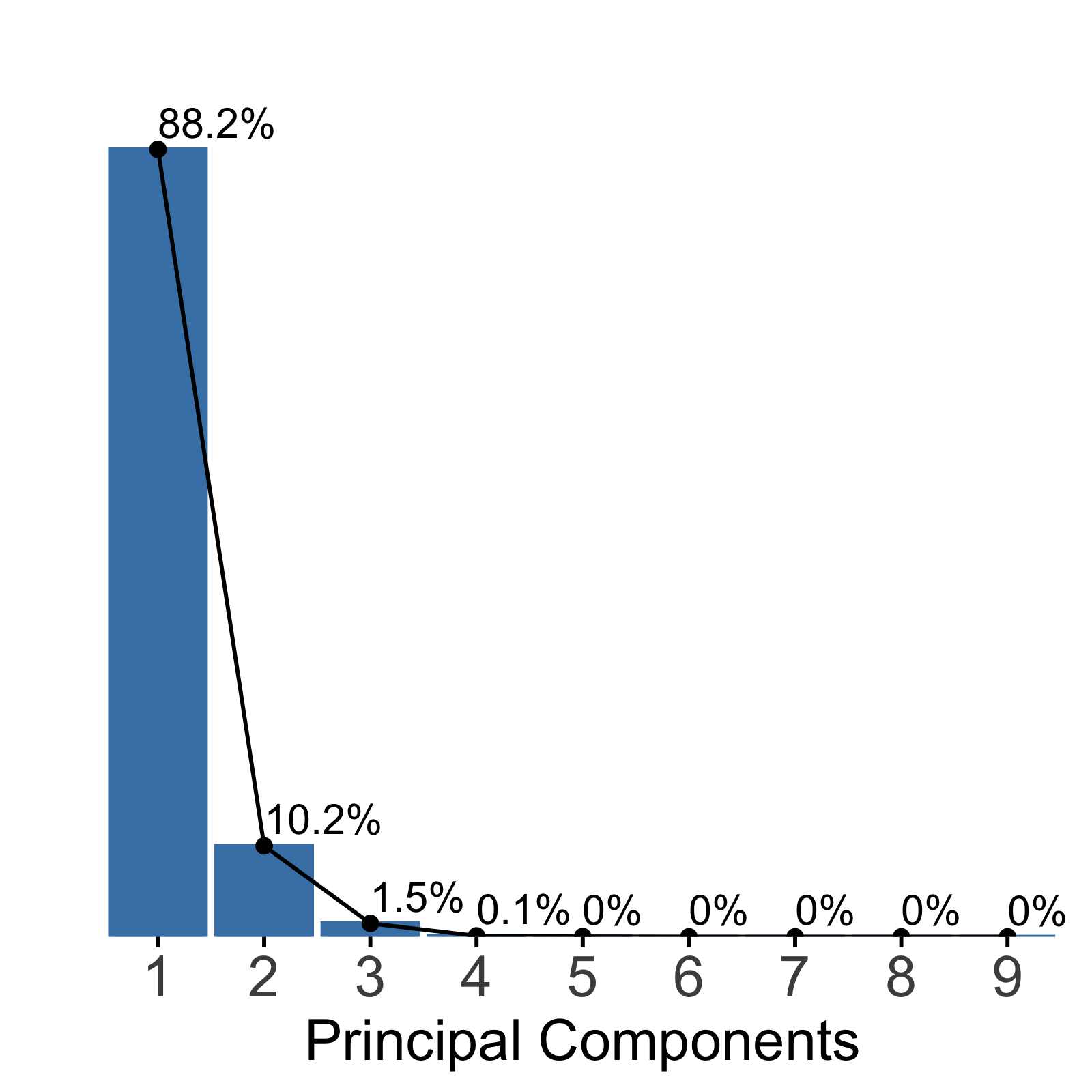} \caption{Republican}
	    \end{subfigure}
	\caption{Scree plots for the Two-party, Democrat, and Republican network samples in the Senate co-voting study. \label{fig:political_scree}}
\end{figure}

\begin{figure}[htbp!] 
    \centering
    \begin{subfigure}[t]{\textwidth} 
        \centering
        \includegraphics[width = \textwidth]{Whole2.png} \caption{Two-party Network}
    \end{subfigure} \hfill
    \begin{subfigure}[t]{\textwidth} 
        \centering
      \includegraphics[width = \textwidth]{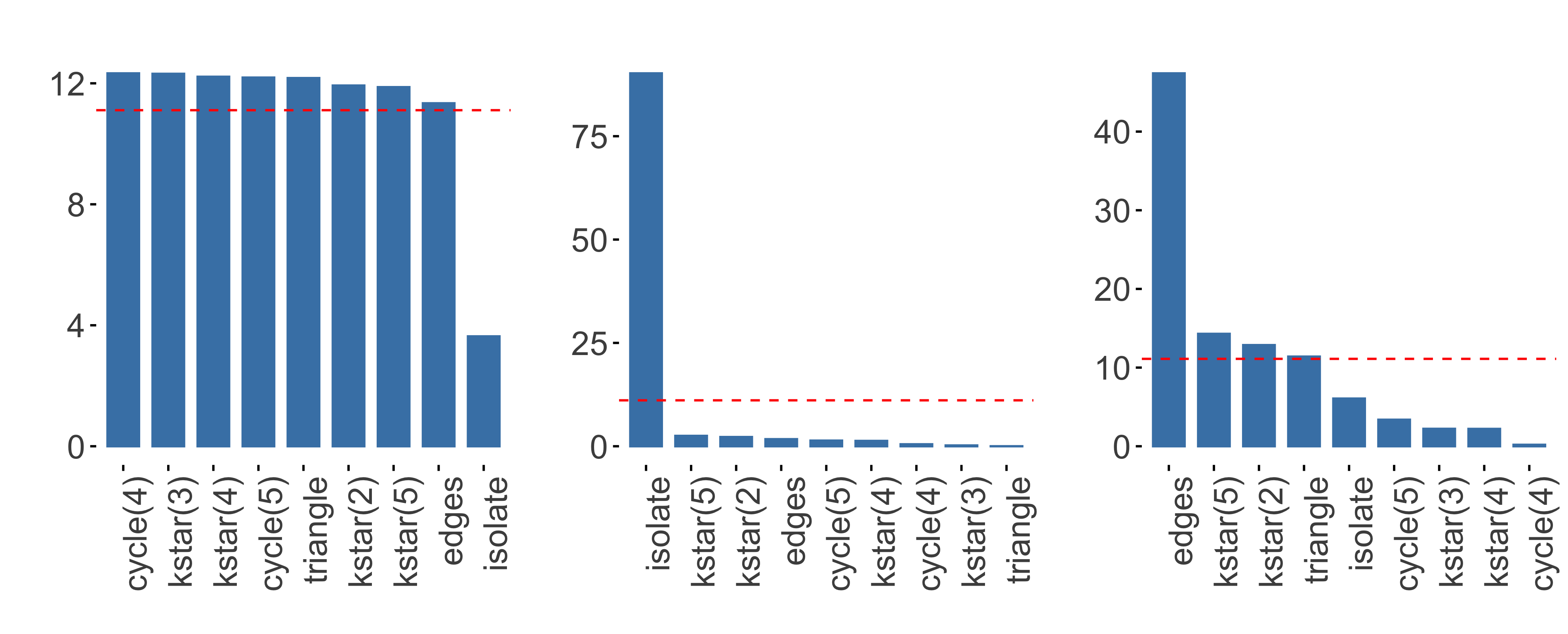} \caption{Democrat Network}
    \end{subfigure} \hfill
    \begin{subfigure}[t]{\textwidth} 
	        \centering
	\includegraphics[width = \textwidth]{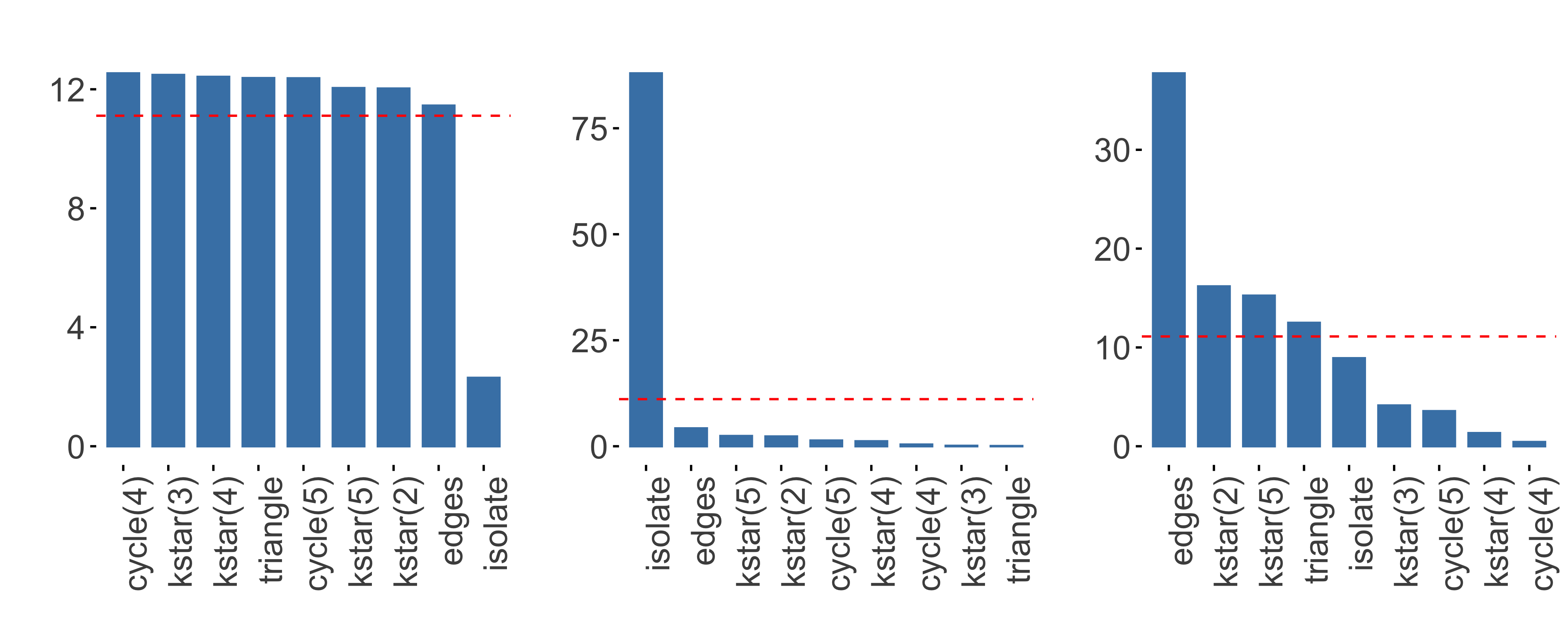} \caption{Republican Network}
	    \end{subfigure}
	\caption{Percentage contributions of the network configurations on the top three principal components for the Two-party, Republican, and Democrat network samples in the Senate co-voting application. The dotted red line shows the average contribution ($= 1/9 * 100\%$) across all 9 configurations. \label{fig:political_contrib}}
\end{figure}

Next, we analyzed the distributions of the scores for each network sample. We split the score distributions into those corresponding to the "polarized" era and those corresponding to the "non-polarized" era. The distributions for the top three PCs are shown in Figure \ref{fig:political_scores}. The plots of these distributions reveal a simple and clear message: the first two principal components of each network very clearly discriminate between the polarized and non-polarized eras. 

\begin{figure}[htbp!] 
    \centering
    \begin{subfigure}[t]{\textwidth} 
        \centering
        \includegraphics[width = \textwidth]{Whole3.png} \caption{Two-party Network}
    \end{subfigure} \hfill
    \begin{subfigure}[t]{\textwidth} 
        \centering
      \includegraphics[width = \textwidth]{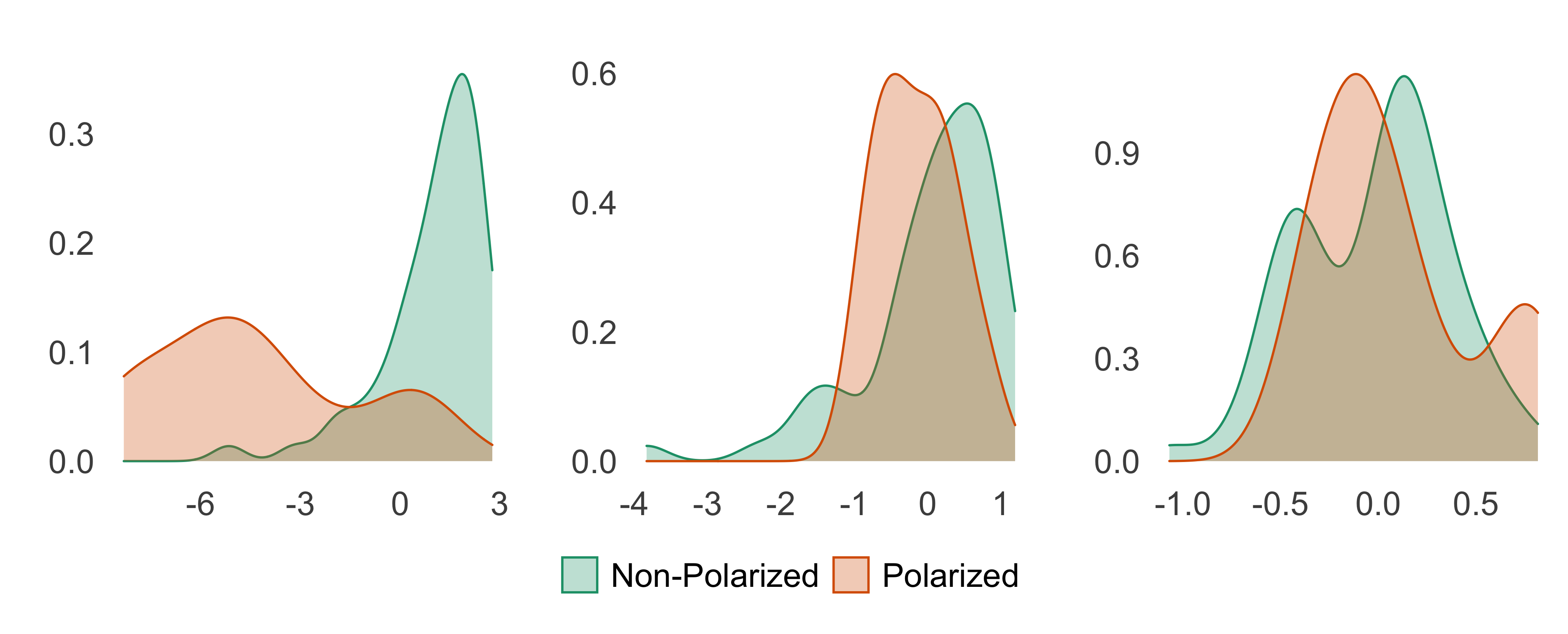} \caption{Democrat Network}
    \end{subfigure} \hfill
    \begin{subfigure}[t]{\textwidth} 
	        \centering
	\includegraphics[width = \textwidth]{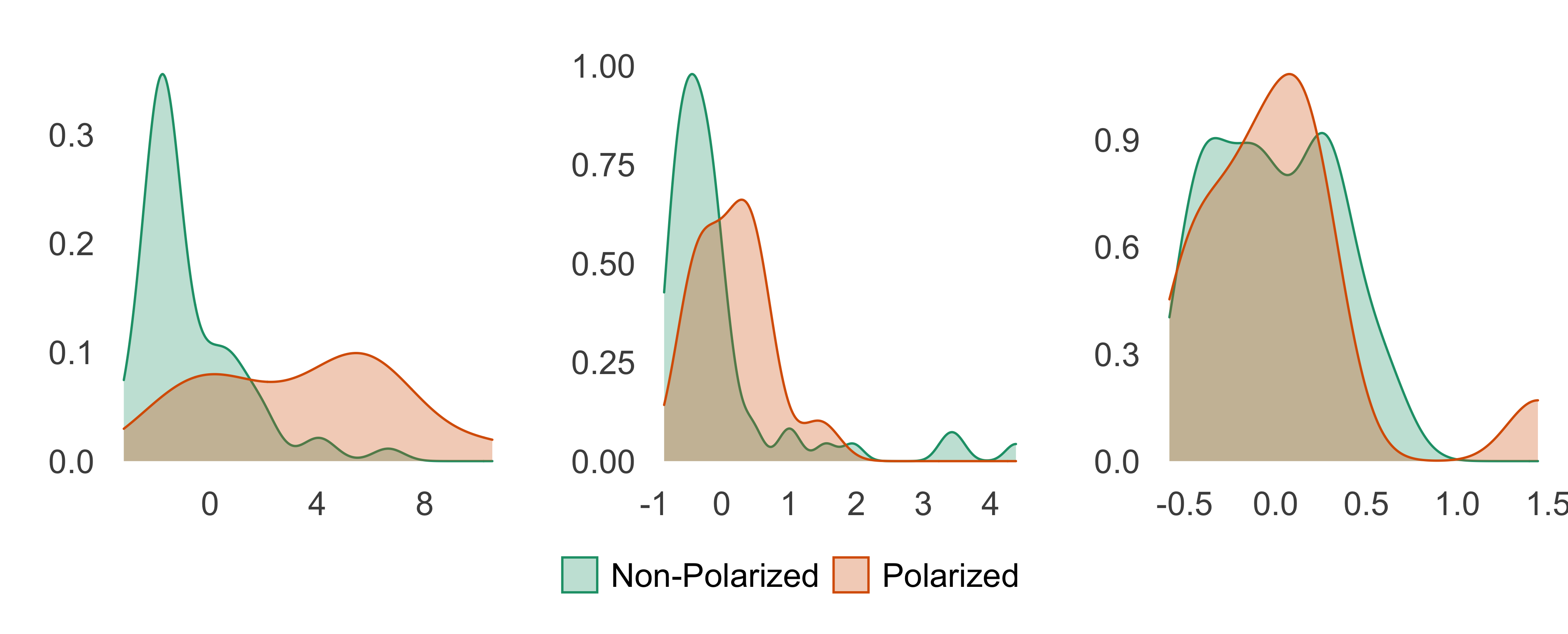} \caption{Republican Network}
	    \end{subfigure}
	\caption{The distributions of the principal component scores for the polarization era and the non-polarization era for the top three principal components. Shown are the distributions for the Two-party, Democrat, and Republican network sample.\label{fig:political_scores}}
\end{figure}

\subsubsection{Predictive Analysis}

To fully assess the predictive ability of the principal components, we fit logistic regression models on the binary outcome of polarized or non-polarized eras. We fit four different models: two models each for whole network (two-party) and the configurations across each party separately (Republican or Democrat), one of which contained the top three PCs as features and the other which contained the raw subgraph configuration densities for all nine configurations considered. We reported the results of the principal component models in Tables \ref{tab:political_whole} (two-party) and \ref{tab:political_subnetwork} (Republican and Democrat). As expected from our exploratory analysis, we found that the first PC was statistically significant in the model, suggesting that the first PC was predictive of senate polarization. No other PCs were statistically significant across the network samples. Notably, the AUC was greater than 94$\%$ in all four models. These results revealed that there were dramatic differences in the network topology of the co-voting networks before and during the polarization eras. Altogether, our analysis suggests that the differences in the polarized and non-polarized eras are dramatic, and furthermore that these differences can largely be captured in just one or two dimensions.

\begin{table}[htbp!]
\centering
\begin{tabular}{l|rrrr}
\toprule
\multicolumn{1}{c|}{PC} & \multicolumn{1}{c}{Variability} & \multicolumn{1}{c}{OR} & \multicolumn{1}{c}{95\% CI} & \multicolumn{1}{c}{p-value} \\
\hline
PC1                    & 88.6\%                          & 2.66                   & (1.75, 5.25)                & $\mathbf{<.001}$               \\
PC2                    & 9.8\%                           & 1.58                   & (0.07, 5.22)                & 0.599                      \\
PC3                    & 1.5\%                           & 0.43                   & (0.01, 20.75)               & 0.682                     \\   
\bottomrule
\end{tabular} 
\caption{A summary of the logistic regression model fit on the principal components identified for the two party to predict whether each co-voting network was in the polarized era or non-polarized era (AUC = 0.946; AIC = 36.9; BIC = 46.1). \label{tab:political_whole}}
\end{table}

\begin{table}[htbp!]
\centering
\begin{tabular}{l|l|rrrr}
\toprule
\multicolumn{1}{c|}{Subnetwork} & \multicolumn{1}{c|}{PC} & \multicolumn{1}{c}{Variability} & \multicolumn{1}{c}{OR} & 95\% CI                   & \multicolumn{1}{c}{p-value} \\
\hline
Democrat   & PC1 & 89.7\%                          & 0.39                   & (0.15, 0.66)                & {\bf 0.008}                       \\
           & PC2 & 8.6\%                           & 1                      & (0.15, 20.53)               & 0.997                       \\
           & PC3 & 1.6\%                           & 2.54                   & (0.03, 446.07)              & 0.696                       \\ \hline
Republican  & PC1 & 88.2\%                          & 1.37                   & (0.91, 2.22)                & 0.141                       \\
           & PC2 & 10.2\%                          & 0.47                   & (0.07, 1.49)                & 0.284                       \\
           & PC3 & 1.5\%                           & 2.86                   & (0.05, 149.62)              & 0.593                      \\
\bottomrule
\end{tabular}
\caption{A summary of the logistic regression model fit on the principal components identified for the Republican and Democrat to predict whether each co-voting network was in the polarized era or non-polarized era (AUC = 0.944; AIC = 41.6; BIC = 57.8). Features in each model were the top three principal components identified for the network sample. \label{tab:political_subnetwork}}
\end{table}

\subsection{The Relationship of sPCAN and PCAN}\label{sec:study3}

We next evaluated the relationship of the sPCAN method with PCAN. We ran sPCAN on both the functional connectivity whole-brain network sample and the two-party co-voting sample to assess its performance and compare with the output of the PCAN algorithm. For each application, we chose the smallest $\tau$ to ensure that the constraint in (\ref{eq:constraint}) was true. This resulted in $\tau = 12$ and $K$ was set to 5 for the co-voting network sample and 22 for the functional connectivity sample. For each network sample, we ran sPCAN 100 times and evaluated (i) the computational speed of sPCAN, (ii) the relationship between the principal component loadings from sPCAN with PCAN, and (iii) the relationship in the scree plots for each method. 

As expected, sPCAN ran much faster than the PCAN method for each application. sPCAN took 28.99 seconds on the whole-brain network sample and 5.01 seconds on the co-voting data, where PCAN took 390.46 seconds and 39.51 seconds, respectively. All computations were performed on a standard laptop with 16 GB memory. We plotted the loadings and scree plots from both methods in Figure \ref{fig:sPCAN_plots}. The scree plots (first column of Figure \ref{fig:sPCAN_plots}) show that sPCAN identified PCs whose variability explained closely matched those identified by PCAN. Furthermore, the PC loadings of the first two PCs (right two columns of Figure \ref{fig:sPCAN_plots}) also closely agree between the two methods. We note that there are some differences in magnitude of the PC loadings, but that the ranking of the important configurations stay roughly the same. These fluctuations are expected due to the complexities of the network samples and their departure from the kernel-based random graph model as well as the small value of $K$ in these examples. Overall, the close alignment is promising and suggests that in the case of large networks, sPCAN provides a fast approximate alternative to the PCAN method.

\begin{figure}[htbp!]
    \centering
    \begin{subfigure}[t]{0.31\textwidth} 
        \centering
        \includegraphics[width = 0.8\linewidth]{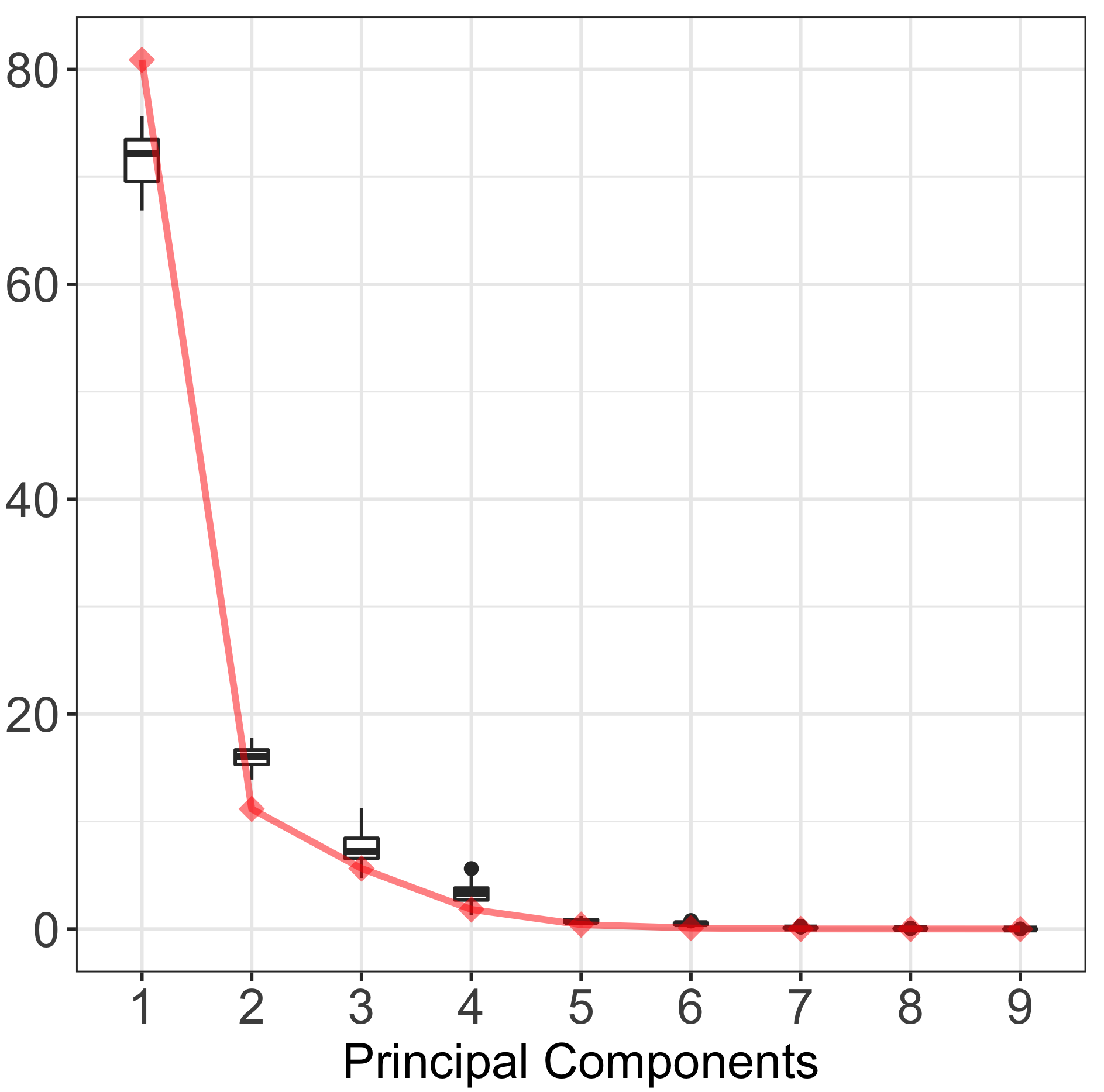} \caption{fMRI: Scree Plot} \label{fig:fMRI_scree}
    \end{subfigure} \hfill
        \begin{subfigure}[t]{0.31\textwidth} 
        \centering
        \includegraphics[width = 0.8\linewidth]{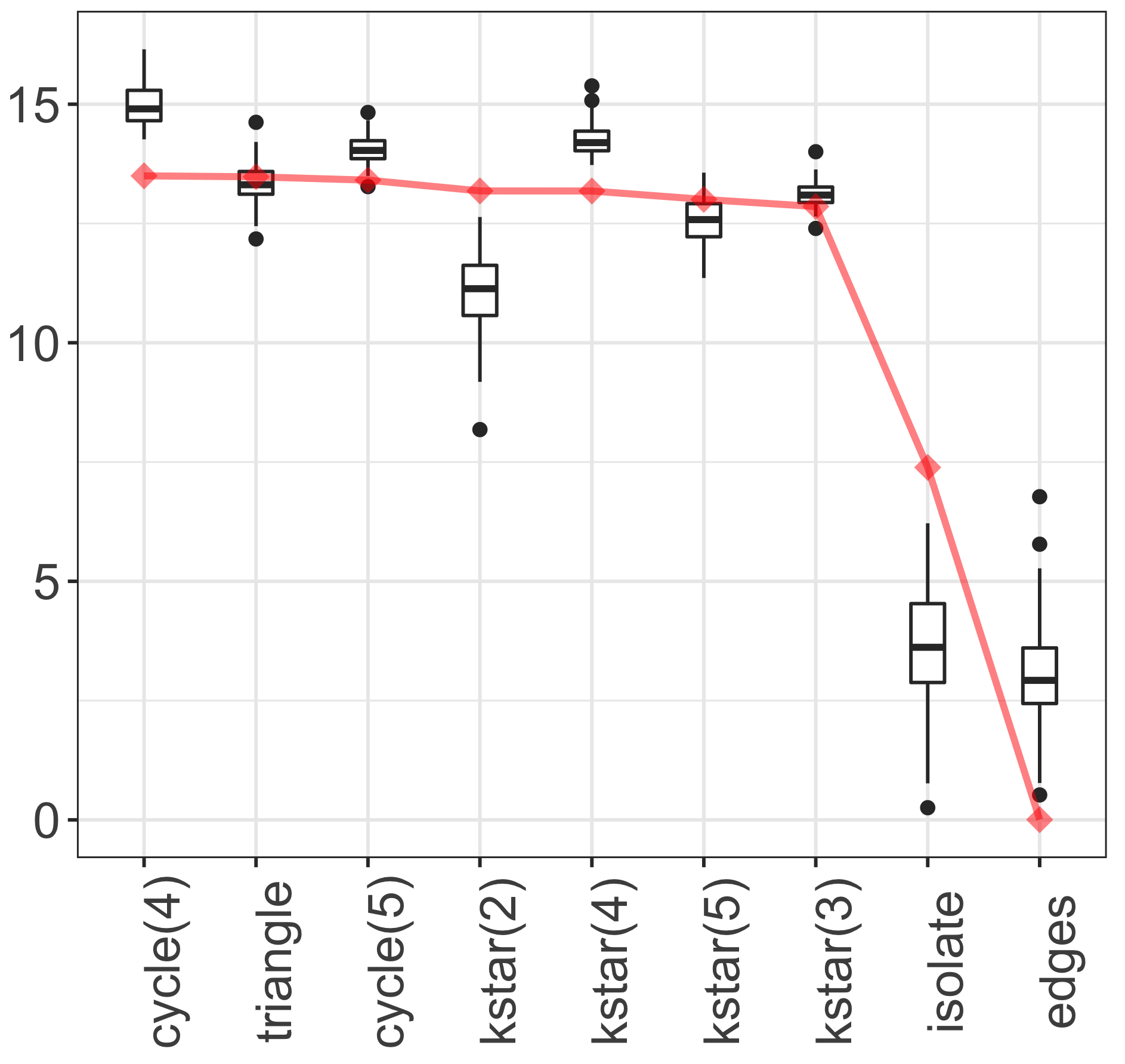} \caption{fMRI: PC1}\label{fig:fMRI_PC1}
    \end{subfigure} \hfill
        \begin{subfigure}[t]{0.31\textwidth} 
        \centering
        \includegraphics[width = 0.8\linewidth]{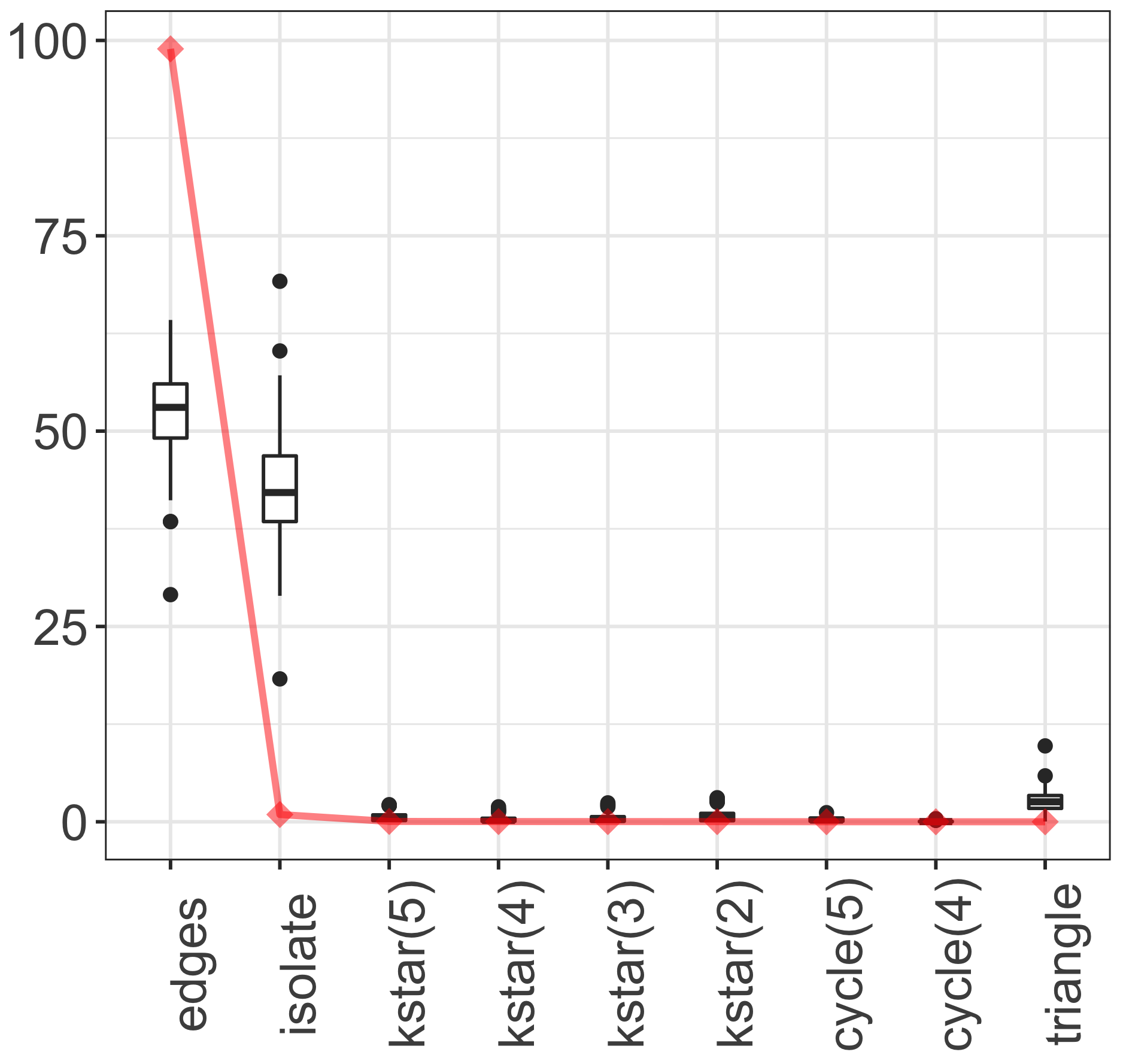} \caption{fMRI: PC2}\label{fig:fMRI_PC2}
    \end{subfigure} \hfill
    \begin{subfigure}[t]{0.31\textwidth} 
        \centering
        \includegraphics[width = 0.8\linewidth]{sPCAN_vs_PCAN_scree_plot.png} \caption{Co-voting: Scree Plot}\label{fig:fMRI_scree}
    \end{subfigure} \hfill
        \begin{subfigure}[t]{0.31\textwidth} 
        \centering
        \includegraphics[width = 0.8\linewidth]{sPCAN_vs_PCAN_PC_1.png} \caption{Co-voting: PC1}\label{fig:fMRI_PC1}
    \end{subfigure} \hfill
        \begin{subfigure}[t]{0.31\textwidth} 
        \centering
        \includegraphics[width = 0.8\linewidth]{sPCAN_vs_PCAN_PC_2.png} \caption{Co-voting: PC2}\label{fig:fMRI_PC2}
    \end{subfigure} \hfill
    \caption{Comparison of sPCAN results with PCAN results on applications. Points in red represent results from PCAN and boxplots show the results from running sPCAN 100 times.} \label{fig:sPCAN_plots}
\end{figure}

\section{Discussion}\label{sec:conclusion}

Despite the increasing prominence of network samples arising in nature, the tools available for analyzing network samples remain limited. In this paper, we considered the problem of interpretable network representation learning for network samples. We first introduced a technique named PCAN that identifies statistically meaningful low-dimensional representations of a network sample via subgraph count statistics. We furthermore developed a fast sampling-based procedure, sPCAN, that not only is significantly more efficient than PCAN, but enjoys the same advantages of interpretability. Large-sample analyses of the methods under kernel-based random graphs revealed that the principal components identified by the sPCAN and PCAN methods are asymptotically equivalent, and furthermore that the embeddings of sPCAN enjoy a central limit theorem. 

Unlike contemporary network representation methods, the results PCAN and sPCAN are directly interpretable and have meaning. Well-understood summary statistics of the network sample can be closely approximated using the principal components identified by the methods, and the overall variability of a network sample can be assessed. We highlighted the utility of the PCAN method with applications to two distinct case studies. Not only were we able to identify meaningful variation in the network samples considered, but we were also able to construct highly accurate predictive models using the principal components as features. 

Our work here motivates several avenues of future research. By adapting principal components to the network setting, we developed a strategy for identifying uncorrelated features for a network or network sample. Exponential random graph models (ERGMs) have long struggled with issues of model degeneracy, partly because of correlation of the features incorporated in the model \citep{handcock2003assessing}. One may be able to use principal components as features in the ERGM and avoid the issue of model degeneracy entirely. This strategy would be akin to principal component regression for network-valued data. We plan to explore this in future research. Furthermore, equation (\ref{eq:approx}) directly enables the approximation of subgraph density features in a network using the embeddings identified by PCAN. To the best of our knowledge, this is the first such network representation learning technique to relate embeddings back to well-understood network summaries \citep{seshadhri2020impossibility}. Though not considered here, an analysis of the accuracy of this approximation would be useful for applications where some networks in a sample are not provided due to privacy reasons. Finally, it would be interesting to pursue the use of principal components in the context of the network monitoring problem \citep{woodall2017overview, jeske2018statistical}. By evaluating shifts in the PCs for graphs in a dynamic sequence, one could formulate a natural nonparametric strategy for network change point detection. In summary, we believe that the PCAN and sPCAN methods set the stage for a new line of inquiry in interpretable network representation learning and look forward to future research in this area. 


\section*{Appendix}

\subsection*{Description of the COBRE data set and scanning parameters}\label{sec:appendix1}

The original COBRE dataset includes functional connectivity matrices for 72 patients with schizophrenia and 75 healthy controls (HC). Exclusion criteria for both groups were history of neurological disorder, mental retardation, severe head trauma (more than 5 minutes of unconsciousness), and substance abuse or dependence within the last 12 months. Diagnostic information was collected using the Structured Clinical Interview for the Diagnostic and Statistical Manual, 4th edition. We additionally removed participants who were diagnosed with disorders other than schizophrenia or schizoaffective disorders. 

The participants' structural MRI scans were completed using a multi-echo MPRAGE sequence with the following parameters: TR/TE/TI = 2530/[1.64, 3.5, 5.36, 7.22, 9.08]/900 ms, flip angle = 7 degree, FOV = 256x256 mm, Slab thickness = 176 mm, Matrix = 256x256x176, Voxel size = 1x1x1 mm, Number of echos = 5, Pixel bandwidth = 650 Hz, Total scan time = 6 min. Given five echoes the multi-echo MPRAGE's TR, TI and time to encode partitions are similar to that of a conventional MPRAGE granting similar gray matter, white matter (WM), and cerebrospinal fluid (CSF) contrast. Slices were collected interleaved in the sagittal plane (multi-slice mode, single shot).

Resting state data was collected with single-shot full k-space echo-planar imaging (EPI) with ramp sampling correction using the intercommissural line (AC-PC) as a reference (TR: 2 s, TE: 29 ms, matrix size: 64x64, 32 slices, voxel size: 3x3x4 mm3). Slices were collected ascending in the axial plane (multi-slice mode, series interleaved). 

Scans were preprocessed using FSL (FMRIB Software Library, www.fmrib.ox.ac.uk/fsl) and FEAT (FMRI Expert Analysis Tool) Version 6.00. Scans were motion corrected using MCFLIRT \cite{Jenkinson2001}, non-brain stripping with BET \cite{Smith2002}, spatial smoothing using a Gaussian kernel of FWHM 6 mm, grand-mean intensity normalization of the entire 4D dataset using a single multiplicative factor, and high-pass temporal filtering (Gaussian-weighted least-squares straight line fitting, sigma = 45 s). Finally, nonlinear registration was completed to Montreal-Neurological Institute space using FNIRT \cite{Jenkinson2002, Jenkinson2001}.

\end{document}